\documentclass[sigconf]{acmart}
\usepackage{amsmath,amsfonts,amssymb,xcolor}
\usepackage{enumitem,graphicx}
\usepackage{algorithm,algorithmic}
\usepackage{pgfplots}
\usepackage{pgfplotstable}
\usepgfplotslibrary{groupplots,units}
\usepackage{microtype}
\usepackage{graphicx}
\usepackage{subfigure}
\usepackage{hyperref}
\usepackage{etoolbox}
\usepackage{booktabs}
\usepackage{inputenc}
\usepackage{amsthm}
\usepackage{amsmath}

\DeclareMathOperator*{\argmin}{arg\,min}
\usepackage{amsfonts}
\usepackage{bm}

\newtheorem{theorem}{Theorem}

\newtheorem{assumption}{Assumption}
\newtheorem{definition}{Definition}
\usepackage{tikz}
\usepackage{hyperref}
\usepackage{color}
\usepackage{natbib}
\usepackage{float}
\hypersetup{%
colorlinks=false,
citecolor=black,
colorlinks,
linkcolor=Red,
urlcolor=Blue}
\def\BibTeX{{\rm B\kern-.05em{\sc i\kern-.025em b}\kern-.08emT\kern-.1667em\lower.7ex\hbox{E}\kern-.125emX}}

\usepackage{balance}
\copyrightyear{2019} 
\acmYear{2019} 
\setcopyright{acmcopyright}
\acmConference[KDD '19]{The 25th ACM SIGKDD Conference on Knowledge Discovery and Data Mining}{August 4--8, 2019}{Anchorage, AK, USA}
\acmBooktitle{The 25th ACM SIGKDD Conference on Knowledge Discovery and Data Mining (KDD '19), August 4--8, 2019, Anchorage, AK, USA}
\acmPrice{\$15.00}
\acmDOI{10.1145/3292500.3330915}
\acmISBN{978-1-4503-6201-6/19/08}

\begin{document}

\title{Dual Averaging Method for Online Graph-structured Sparsity}

\author{Baojian Zhou}
\email{bzhou6@albany.edu}
\affiliation{%
  \institution{University at Albany, SUNY}
  \streetaddress{1400 Washington Avenue}
  \city{Albany}
  \state{NY, USA}
  \postcode{12222}
}

\author{Feng Chen}
\email{fchen5@albany.edu}
\affiliation{%
  \institution{University at Albany, SUNY}
  \streetaddress{1400 Washington Avenue}
  \city{Albany}
  \state{NY, USA}
  \postcode{12222}
}

\author{Yiming Ying}
\email{yying@albany.edu}
\affiliation{%
  \institution{University at Albany, SUNY}
  \streetaddress{1400 Washington Avenue}
  \city{Albany}
  \state{NY, USA}
  \postcode{12222}
}

%%%%%%%%%%%%%%%%%%%%%%%%%%%%%%%%%%%%%%%%%%%%%%%%%%%%%%%%%%%%%%%%%%%%%%%%%%%%%%%%
\begin{abstract}
Online learning algorithms update models via one sample per iteration, thus efficient to process large-scale datasets and useful to detect malicious events for social benefits, such as disease outbreak and traffic congestion on the fly. However, existing algorithms for graph-structured models focused on the  offline setting and the least square loss, incapable for online setting, while methods designed for online setting cannot be directly applied to the problem of complex (usually non-convex) graph-structured sparsity model. To address these limitations, in this paper we propose a new algorithm for graph-structured sparsity constraint problems under online setting, which we call \textsc{GraphDA}. The key part in \textsc{GraphDA} is to project both averaging gradient (in dual space) and primal variables (in primal space) onto lower dimensional subspaces, thus capturing the graph-structured sparsity effectively. Furthermore, the objective functions assumed here are generally convex so as to handle different losses for online learning settings. To the best of our knowledge, \textsc{GraphDA} is the first online learning algorithm for graph-structure constrained optimization problems. To validate our method, we conduct extensive experiments on both benchmark graph and real-world graph datasets. Our experiment results show that, compared to other baseline methods, \textsc{GraphDA} not only improves classification performance,  but also successfully captures graph-structured features more effectively, hence stronger interpretability.
\end{abstract}


\maketitle

%%%%%%%%%%%%%%%%%%%%%%%%%%%%%%%%%%%%%%%%%%%%%%%%%%%%%%%%%%%%%%%%%%%%%%%%%%%%%%%%
\section{Introduction}
\label{section:introduction}

As a new paradigm in machine learning, convex online learning algorithms have received enormous attention~\cite{ying2006online,duchi2011adaptive,xiao2010dual,hazan2016introduction,shalev2012online,kingma2014adam,ying2017unregularized}. These algorithms update learning models sequentially by using one training sample at each iteration, which makes them applicable to large-scale datasets on the fly and still enjoy \textit{non-regret} property. For better interpretability and less computational complexity in high dimension data, many online learning algorithms~\cite{xiao2010dual,yang2010online,langford2009sparse,duchi2011adaptive} exploit $\ell_1$ norm or $\ell_1/\ell_2$ mixed norm to achieve sparse solution~\cite{xiao2010dual,yang2010online,jacob2009group}. However, these sparsity-inducing models cannot characterize more complex (usually non-convex) graph-structured sparsity constraint, hence, unable to use some important priors such as graph data. 

Graph-structured sparsity models have significant real-world applications, for example, social events~\cite{rozenshtein2014event}, disease outbreaks~\cite{qian2014connected}, computer viruses~\cite{draief2006thresholds}, and gene networks~\cite{chuang2007network}. These applications all contain graph structure information and the data samples are usually collected on the fly, i.e., the training samples have been received and processed one by one. Unfortunately, most of the graph-structured (non-convex) methods~\cite{hegde2015nearly,hegde2016fast,chen2016generalized,aksoylar2017connected} are batch learning-based, which cannot be applied to the online setting. The past few years have seen a surge of convex online learning algorithms, such as online projected gradient descent~\cite{zinkevich2003online}, \textsc{AdaGrad}~\cite{duchi2011adaptive}, \textsc{Adam}~\cite{kingma2014adam}, \textsc{$\ell_1$-RDA}~\cite{xiao2010dual}, \textsc{FOBOS}~\cite{duchi2009efficient}, and many others (e.g.~\cite{hazan2016introduction,shalev2012online}). However, they cannot be used to tackle online graph-structured sparsity problems due to the limitation of sparsity-inducing norms.

In recent years, machine learning community~\cite{lafond2015online,hazan2017efficient,chen2018online,gonen2018learning,gao2018online,yang2018optimal} have made promising progress on online non-convex optimization with regards to algorithms and \textit{local-regret} bounds.  Nonetheless, these algorithms cannot deal with graph-structured sparsity constraint problems due to the following two limitations: 1) The existing non-convexity assumption is only on the loss functions subject to a convex constraint; 2) Most of these proposed algorithms are based on online projected gradient descent (PGD), and cannot explore the structure information, hardly workable for graph-structured sparsity constraint. To the best of our knowledge, there is no existing work to tackle the combinatorial non-convexity constraint problems under online setting. 

In this paper, we aim to design an approximated online learning algorithm that can capture graph-structured information effectively and efficiently. To address this new and challenging question, the potential algorithm has to meet two crucial requirements: 1) \textit{graph-structured:} The algorithm should effectively capture the latent graph-structured information such as trees, clusters, connected subgraphs; 2) \textit{online:} The algorithm should be efficiently applicable to online setting where training samples can only be processed one by one. Our assumption on the problem has a non-convex constraint but with a convex objective, which will sustain higher applicability in the practice of our setting. Inspired by the success of dual-averaging~\cite{nesterov2009primal,xiao2010dual}, we propose the Graph Dual Averaging Algorithm, namely, \textsc{GraphDA}. The key part in \textsc{GraphDA} is to keep track of both averaging gradient via dual variables in dual space and primal variables in primal space.  We then use two approximated projections to project both primal variables and dual variables onto low dimension subspaces at each iteration. We conduct extensive experiments to demonstrate that by projecting both primal and dual variables, \textsc{GraphDA} captures the graph-structured sparsity effectively. Overall, our contributions are as follows:

\noindent$\bullet$ We propose a dual averaging-based algorithm to solve graph-structured sparsity constraint problems under online setting. To the best of our knowledge, it is a first attempt to establish an online learning algorithm for the graph-structured sparsity model.

\noindent$\bullet$ We prove the minimization problem occurring at each dual averaging step, which can be formulated as two equivalent optimization problems: minimization problem in primal space and maximization problem in dual space. The two optimization problems can then be solved approximately by adopting two popular projections. Furthermore, we provide two exact projection algorithms for the non-graph data.

\noindent$\bullet$ We conduct extensive experiments on both synthetic and real-world graphs. The experimental results demonstrate that \textsc{GraphDA} can successfully capture the latent graph-structure during online learning process. The learned model generated by our algorithm not only achieves higher classification accuracy but also stronger interpretability compared with the state-of-the-art algorithms.

The rest of the paper is organized as follows: Related work is teased out in Section~\ref{section:relate-work}. Section~\ref{section:preliminary} gives the notations and problem definition.  In Section~\ref{section:algorithms}, we present our main idea and algorithms. We report and discuss the experiment results in comparison with other baseline methods in Section~\ref{section:experiments}. A short conclusion ensues in Section~\ref{section:conclusion}. Due to space limit, the detailed experimental setup and partial experimental results are supplied in Appendix. Our source code including baseline methods and datasets are accessible at: \textcolor{blue}{\url{https://github.com/baojianzhou/graph-da}}. 

%%%%%%%%%%%%%%%%%%%%%%%%%%%%%%%%%%%%%%%%%%%%%%%%%%%%%%%%%%%%%%%%%%%%%%%%%%%%%%%%
\section{Related Work}
\label{section:relate-work}

In line with the focus of the present work, we categorize highly related researches into three sub-topics for the sake of clarity.

\textbf{Online learning with sparsity.\quad} Online learning algorithms~\cite{zinkevich2003online, bottou2004large,ying2008online,hazan2016introduction,shalev2012online} try to solve classification or regression problems that can be employed in a fully incremental fashion.  A natural way to solve online learning problem is to use stochastic gradient descent by using one sample at a time. However, this type of methods usually cannot produce any sparse solution. The gradient of only one sample has such a large variance that renders its projection unreliable. To capture the model sparsity, $\ell_1$ norm-based~\cite{bottou1998online,duchi2009efficient,duchi2011adaptive,langford2009sparse,xiao2010dual} and $\ell_1/\ell_2$ mixed norm-based~\cite{yang2010online} are used  under online learning setting; the dual-averaging~\cite{xiao2010dual} adds a convex regularization, namely \textsc{$\ell_1$-RDA} to learn a sparsity model. Based on the dual averaging work, online group lasso and overlapping group lasso are proposed in~\cite{yang2010online}, which provides us a sparse solution. However, the solution cannot produce methods directly applicable to graph-structured data. For example, as pointed out by~\cite{xiao2010dual}, the levels of sparsity proposed in \cite{duchi2009efficient,langford2009sparse} are not satisfactory compared with their batch counterparts.

\textbf{Model-based sparsity.\quad} Different from $\ell_1$-regularization~\cite{tibshirani1996regression} or $\ell_1$-ball constraint-based method~\cite{duchi2008efficient}, model-based sparsity are non-convex~\cite{baraniuk2010model,hegde2015approximation,hegde2015nearly,hegde2016fast}. Using non-convex such as $\ell_0$ sparsity based methods~\cite{yuan2014gradient,bahmani2013greedy,zhou2018efficient,nguyen2017linear} becomes popular, where the objective function is assumed to be convex with a sparsity constraint. To capture graph-structured sparsity constraint such as trees and connected graphs, a series of work ~\cite{baraniuk2010model,hegde2014fast,hegde2016fast,hegde2015nearly} has proposed to use structured sparsity model $\mathbb{M}$ to define allowed supports $\mathbb{M}=\{S_1,S_2,\ldots,S_k\}$. These complex models are non-convex, and gradient descent-based algorithms involve a projection operator which is usually NP-hard. \cite{hegde2015nearly,hegde2016fast,hegde2015approximation} use two approximated projections (head and tail) without sacrificing too much precision. However, the above research work cannot be directly applied to online setting and the objective function considered is not general loss.

\textbf{Online non-convex optimization.\quad} The basic assumption in recent progress on online non-convex optimization~\cite{lafond2015online,hazan2017efficient,chen2018online,gonen2018learning,gao2018online,yang2018optimal} is that the objective considered is non-convex. \textit{Local-regret} bound has been explored in these studies, most of which are based on projected gradient descent methods, for example, online projected gradient descent~\cite{hazan2017efficient} and online normalized gradient descent~\cite{gao2018online}. However, these online non-convex algorithms cannot deal with our problem setting where there exists a combinatorial non-convex structure.

%%%%%%%%%%%%%%%%%%%%%%%%%%%%%%%%%%%%%%%%%%%%%%%%%%%%%%%%%%%%%%%%%%%%%%%%%%%%%%%%
\section{Preliminaries}
\label{section:preliminary}

We begin by introducing some basic mathematical terms and notations, and then define our problem. 
\subsection{Notations}
An index set is defined as $[p]=\{1,\ldots,p\}$. The bolded lower-case letters, e.g., ${\bm w}, {\bm x} \in \mathbb{R}^p$, denote column vectors where their $i$-th entries are $w_i$, $x_i$. The $\ell_2$-norm of ${\bm w}$ is denoted as $\|{\bm w}\|_2$. The inner product of ${\bm x}$ and ${\bm y}$ on $\mathbb{R}^p$ is defined as $\langle {\bm x}, {\bm y} \rangle = x_1 y_1 + \cdots + x_p y_p$. Given a differentiable function $f({\bm w}): \mathbb{R}^p \rightarrow \mathbb{R}$, the gradient at ${\bm w}$ is denoted as $\nabla f({\bm w})$. The support set of ${\bm w}$, i.e., $\text{supp}({\bm w}):= \{i|w_i \ne 0\}$, is defined as a subset of indices which index non-zero entries. If $|\text{supp}({\bm w})| \leq s$, ${\bm w}$ is called an $s$ sparse vector. The upper-case letters, e.g., $\Omega$, denote a subset of $[p]$ and its complement is $\Omega^c=[p]\backslash \Omega$. The restricted vector of ${\bm w}$ on $\Omega$ is denoted as ${\bm w}_{\Omega} \in \mathbb{R}^p$, where $({\bm w}_\Omega)_i = w_i$ if $i\in \Omega$; otherwise 0. We define the undirected graph as $\mathbb{G}(\mathbb{V},\mathbb{E})$, where $\mathbb{V}=[p]$ is the set of nodes and $\mathbb{E}$ is the set of edges such that $\mathbb{E} \subseteq \{(u,v)| u\in \mathbb{V}, v \in \mathbb{V} \}$.  The upper-case letters, e.g., $H, T, S$, stand for subsets of $[p]:=\{1,2,\ldots, p\}$. Given the standard basis $\{{\bm e}_i: 1\leq i\leq p\}$ of $\mathbb{R}^p$, we also use $H, T, S$ to represent subspaces. For example, the subspace $S$ is the subspace spanned by $S$, i.e., $\text{span}\{{\bm e}_i: i \in S \}$. We will clarify the difference only if confusion occurs. 

\subsection{Problem Definition}
Here, we study an online non-convex optimization problem, which is to minimize the regret as defined in the following:
\begin{equation}
    R(T,\mathcal{M}(\mathbb{M})) := \sum_{t=1}^{T} f_t({\bm w}_t, \{{\bm x}_t, y_t\}) - \min_{{\bm w} \in \mathcal{M}(\mathbb{M})}\sum_{t=1}^{T} f_t({\bm w}),
    \label{definition:regret}
\end{equation}
where each $f_t({\bm w}_t,\{{\bm x}_t,y_t\})$ is the loss that a learner predicts an answer for the question ${\bm x}_t$ after receiving the correct answer $y_t$, and $\min_{{\bm w} \in \mathcal{M}(\mathbb{M})} \sum_{t=1}^{T} f_t({\bm w})$ is the minimum loss that the learner can potentially get. To simplify, we assume $f_t$ is convex differentiable. For example, we can use the least square loss $f_t({\bm w}_t,\{{\bm x}_t,y_t\}) = ({\bm w}_t^\top {\bm x}_t - y_t)^2$ for the online linear regression problem and logistic loss $f_t({\bm w}_t,\{{\bm x}_t,y_t\})$ $= \log( 1+ \exp (- y_t \cdot{\bm w}^\top {\bm x}_t))$ for online binary classification problem where $y_t\in \{\pm 1\}$. The goal of the learner is to minimize the regret $R(T,\mathcal{M}(\mathbb{M}))$. Different from the online convex optimization setting in~\cite{shalev2012online,hazan2016introduction}, $\mathcal{M}(\mathbb{M})\subseteq \mathbb{R}^p$ is a generally non-convex set. To capture more complex graph-structured information, in a series of seminal work~\cite{baraniuk2010model,hegde2015approximation,hegde2015nearly}, a structured sparsity model $\mathcal{M}(\mathbb{M})$ is proposed as follows:
\begin{equation}
    \mathcal{M}(\mathbb{M}) := \{{\bm w} | \text{supp}({\bm w}) \subseteq S \text{ for some } S \in \mathbb{M}\},
    \label{equation:model_m}
\end{equation}
where $\mathbb{M}=\{S_1,S_2,\ldots,S_k\}$ is the collection of allowed structure supports with $S_i \in [p]$. Basically, $\mathcal{M}(\mathbb{M})$ is the union of $k$ subspaces. Each subspace is uniquely identified by $S_i$. Definition~(\ref{equation:model_m}) is so general that it captures a broad spectrum of graph-structured sparsity models such as trees~\cite{hegde2014fast}, connected subgraphs~\cite{hegde2015nearly,chen2016generalized}. We mainly focus on the Weighted Graph Model(WGM) proposed in~\cite{hegde2015nearly}. 

\begin{definition}[Weighted Graph Model~\cite{hegde2015nearly}]
Given an underlying graph $\mathbb{G}=(\mathbb{V},\mathbb{E}, {\bm c})$ defined on the coefficients of the unknown vector ${\bm w}$, where $\mathbb{V}=[p]$, $\mathbb{E}\subseteq \mathbb{V}\times \mathbb{V}$ and associated cost vector ${\bm c}$ on edges, then the weighted graph model $(\mathbb{G},s,g,B)$-WGM can be defined as the following set of supports:
\begin{align*}
\mathbb{M}=&\{F: |F|\leq s, \text{ there is an forest } \mathcal{F} \text{ with } \\
&\quad\mathbb{V}_{\mathcal{F}} = F, \gamma(\mathcal{F}) = g, \text{ and } {\bm c}(\mathcal{F})\leq B\},
\end{align*}
where $B$ is the budget on cost of edges in forest $\mathcal{F}$, $\gamma(\mathcal{F})$ is the number of connected component in forest $\mathcal{F}$ denoted as $g$, and $s$ is the sparsity. To clarify, forest $\mathcal{F}$ is the subgraph induced by its nodes set $F$, i.e. $\mathcal{F}:=\mathbb{G}(F,\mathbb{E}')$, where $\mathbb{E}' = \{(u,v): u \in F, v\in F, (u,v) \in \mathbb{E}\}$. ${\bm c}(F)$ is the total edge costs in forest $\mathcal{F}$.
\end{definition}

\begin{figure}[ht!]
\centering
\begin{tikzpicture}
  [scale=.9,auto=left,every node/.style={draw,circle,thick,fill=black!10}]
  \node (n16) at (1,9)  {${w}_6$};
  \node (n14) at (0,8)  {${w}_4$};
  \node[fill=red!50] (n15) at (1,8)  {${w}_5$};
  \node[fill=red!50] (n11) at (2,8)  {${w}_1$};
  \node[fill=red!50] (n12) at (2,7)  {$w_2$};
  \node (n13) at (0,7)  {$w_3$};
  %-------------
  \node[fill=red!50]  (n26) at (4,9)  {$w_6$};
  \node (n24) at (3,8)  {$w_4$};
  \node[fill=red!50]  (n25) at (4,8)  {$w_5$};
  \node (n21) at (5,8)  {$w_1$};
  \node[fill=red!50]  (n22) at (5,7)  {$w_2$};
  \node (n23) at (3,7)  {$w_3$};
  %------------
  \node[fill=red!50]  (n36) at (7,9)  {$w_6$};
  \node[fill=red!50]  (n34) at (6,8)  {$w_4$};
  \node[fill=red!50]  (n35) at (7,8)  {$w_5$};
  \node (n31) at (8,8)  {$w_1$};
  \node (n32) at (8,7)  {$w_2$};
  \node[fill=red!40]  (n33) at (6,7)  {$w_3$};
  \foreach \from/\to in {n16/n14,n14/n15,n12/n13,n13/n14,
  n26/n24,n24/n25,n25/n21,n21/n22,n22/n23,n23/n24,
  n35/n31,n31/n32,n32/n35,n32/n33}
    \draw[line width=0.3mm] (\from) -- (\to);
    \foreach \from/\to in {n15/n11,n11/n12,n12/n15}
    \draw[line width=1.6mm, red!50] (\from) -- (\to);
    \foreach \from/\to in {n22/n25}
    \draw[line width=1.6mm, red!50] (\from) -- (\to);
    \foreach \from/\to in {n36/n34,n34/n35,n33/n34}
    \draw[line width=1.6mm, red!50] (\from) -- (\to);
\draw[dashed] (2.5,6.5) -- (2.5,9.6);
\draw[dashed] (5.5,6.5) -- (5.5,9.6);
\newcommand{\Cross}{$\mathbin{\tikz [x=1.4ex,y=1.4ex,line width=.2ex, red] \draw (0,0) -- (1,1) (0,1) -- (1,0);}$}%
\newcommand{\Checkmark}{$\color{green}\checkmark$}
\node[draw=white,fill=white] at (1.2, 6.3)  {\Checkmark};
\node[draw=white,fill=white] at (4.2, 6.3)  {\Cross};
\node[draw=white,fill=white] at (7.2, 6.3)  {\Cross};
\end{tikzpicture}
\vspace{-5mm}
\caption{A toy example of Weighted Graph Model}
\label{figure:toy_example}
\end{figure}
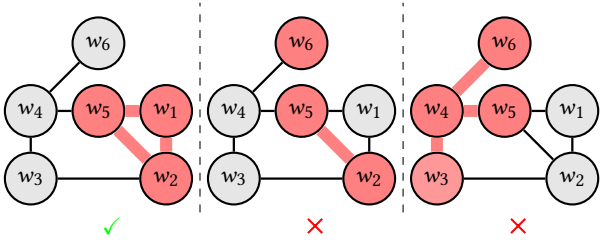

$(\mathbb{G},s,g,B)$-WGM captures a broad range of graph structures such as groups, clusters, trees, and subgraphs. A toy example is given in Figure~\ref{figure:toy_example} where we define a graph with 6 nodes in $\mathbb{V}$, 7 edges in $\mathbb{E}$, and the cost of all edges is set to $1$. Let $\mathbb{V}$ be associated with a vector ${\bm w}\in \mathbb{R}^6$. Suppose we are interested in connected subgraphs\footnote{A connected subgraph is the subgraph which has only 1 connected component.} with at most 3 nodes, to capture these subgraphs, $\mathbb{M}$ can be defined as $\mathbb{M} = \{ S_i| \mathbb{G}(S_i,\mathbb{E}_i) \text{ is connected }, |S_i|\leq 3 \}$. By letting the budget $B=3$, and the sparsity parameter $s=3$, we can clearly use $(\mathbb{G},3,1,3)$-WGM to represent this $\mathbb{M}$. Figure~\ref{figure:toy_example} shows three subgraphs formed by red nodes and edges. The subgraph induced by $\{w_1,w_2,w_5\}$ on the left is in $\mathbb{M}$. However, the subgraph induced by $\{w_2,w_5,w_6\}$ in the middle is not in $\mathbb{M}$ because of the non-connectivity. The subgraph formed by $\{w_3,w_4,w_5,w_6\}$ on the right is not in $\mathbb{M}$ either, as it violates the sparsity constraint, i.e., $s\leq 3$.

After defining the structure-sparsity model $\mathcal{M}(\mathbb{M})$, we explore how to design an efficient and effective algorithm to minimize the regret under model constraint. An intuitive way to do this is to use online projected gradient descent~\cite{zinkevich2003online} where the algorithm needs to solve the following projection at iteration $t$:
\begin{equation}
    {\bm w}_{t+1} = {\rm P}({\bm w}_t - \eta_t \nabla f_{t}({\bm w}_t), \mathcal{M}(\mathbb{M})),
    \label{equation:online_pgd}
\end{equation}
where $\eta_t$ is the learning rate and ${\rm P}$ is the projection operator onto $\mathcal{M}(\mathbb{M})$, i.e., ${\rm P}(\cdot, \mathcal{M}(\mathbb{M})):$ $\mathbb{R}^{p} \rightarrow \mathbb{R}^p$ is defined as
\begin{equation}
    {\rm P}({\bm w},\mathcal{M}(\mathbb{M})) = \argmin_{{\bm x} \in \mathcal{M}(\mathbb{M})} 
    \| {\bm w} - {\bm x} \|^2.
    \label{equ:projection_operator}
\end{equation}
However, there are two essential drawbacks of using~(\ref{equation:online_pgd}): First, the projection in (\ref{equation:online_pgd}) only uses single gradient $\nabla f_t({\bm w}_t)$ which is too noisy (large variance) to capture the graph-structured information at each iteration; Second, the training samples coming later are less important than these coming earlier due to the decay of learning rate $\eta_t$. Recall that $\eta_t$ needs to decay asymptotically to $\mathcal{O}(1/\sqrt{t})$ in order to achieve a non-regret bound. Fortunately, inspired by~\cite{nesterov2009primal,xiao2010dual}, the above two weaknesses can be successfully overcome by using dual averaging. The main idea is to keep tracking both primal vectors (corresponding to ${\bm w}_t$ in primal space) and dual variables (corresponding to gradients, $\nabla f_t({\bm w}_t)$ in dual space\footnote{Notice that we use $\ell_2$-norm, i.e. $\|\cdot\|_2$, which is defined in the Euclidean space $X=\mathbb{R}^p$. By definition, the dual norm of $\ell_2$ is identical to itself, i.e., $\|\cdot\|_* = \|\cdot\|_2$. Also, recall that the dual space $X^*$ of the Euclidean space is also identical with each other $X^*=X=\mathbb{R}^p$.}) at each iteration. In the next section, we focus on designing an efficient algorithm by using the idea of dual averaging to capture graph-structured sparsity under online setting.

%%%%%%%%%%%%%%%%%%%%%%%%%%%%%%%%%%%%%%%%%%%%%%%%%%%%%%%%%%%%%%%%%%%%%%%%%%%%%%%%%%%%%%%%%%%%%%%%%%%%
\section{Algorithm: \textsc{GraphDA}}
\label{section:algorithms}
We try to develop a dual averaging-based method to minimize the regret~(\ref{definition:regret}). At each iteration, the method updates ${\bm w}_t$ by using the following minimization step:
\begin{align}
    {\bm w}_{t+1} &= \argmin_{{\bm w} \in \mathcal{M}(\mathbb{M})} \Bigg\{\frac{1}{t+1} \sum_{i=0}^t \langle {\bm g}_i, {\bm w} \rangle + \frac{\beta_t}{2t} \| {\bm w} \|_2^2 \Bigg\},
    \label{definition:dual_averaging_update}
\end{align}
where $\beta_t$ is to control the learning rate implicitly and ${\bm g}_i$ is a subgradient in $\partial f_i({\bm w},\{{\bm x}_i,y_i\})=\{{\bm g}: f_i({\bm z},\{{\bm x}_i,y_i\}) \geq f_i({\bm w},\{{\bm x}_i,y_i\})+\langle {\bm g}, {\bm z} - {\bm w} \rangle, \forall {\bm z} \in \mathcal{M}(\mathbb{R})\}$\footnote{As we assume $f_i$ is convex differentiable, then we have $\partial f_i({\bm w},\{{\bm x}_i,y_i\}) = \{\nabla f_i({\bm w},\{{\bm x}_i,y_i\})\}$, i.e. ${\bm g}_i =\nabla f_i({\bm w},\{{\bm x}_i,y_i\}). $}. Different from the convexity explored in \cite{nesterov2009primal} and \cite{xiao2010dual}, the problem considered here is generally non-convex, which makes it NP-hard to solve. Initially, the solution of the primal is set to zero, i.e., ${\bm w}_0={\bm 0}$.\footnote{There are two advantages: 1. $\bm w_0 = \bm 0$ is trivially in the $\mathcal{M}(\mathbb{M})$; 2. $\bm w_0 = \bm 0 \in \argmin_{\bm w} \|\bm w\|_2^2$ under convex setting~\cite{xiao2010dual} so that sublinear regret can obtain.} Then at each iteration, it computes a subgradient ${\bm g}_t$ based on current data sample $\{{\bm x}_t, y_t\}$ and then updates ${\bm w}_{t}$ by using~(\ref{definition:dual_averaging_update}) averaging gradient from the dual space. The algorithm terminates after receiving $T$ samples and returns the model ${\bm w}_T$ or $\bar{{\bm w}}_T = 1/T\sum_{t=0}^T {\bm w}_t$ depending on needs. The dual averaging step (\ref{definition:dual_averaging_update}) has two advantages: 1) The gradient information of training samples coming later will not decay when new samples are coming; 2) The averaging gradient can be accumulated during the learning process; hence we can use it to capture graph-structure information more effectively than online PGD-based methods.

Due to the NP-hardness to compute~(\ref{definition:dual_averaging_update}), it is impractical to directly use~(\ref{definition:dual_averaging_update}). Thus, we have to treat this minimization step more carefully for $\mathcal{M}(\mathbb{M})$. The minimization step~(\ref{definition:dual_averaging_update}) has the following equivalent projection problems, specified in Theorem~\ref{theorem:theorem1}.

\begin{theorem}
Assume $\beta_t = \gamma\sqrt{t}$, where $\gamma >0$ and denote $\bar{\bm s}_{t+1} = \frac{1}{t+1} \sum_{i=0}^t {\bm g}_i$. The minimization step of~(\ref{definition:dual_averaging_update}) can be expressed as the following two equivalent optimization problems:
\begin{align}
&\max_{S\in \mathbb{M}} \| P(-\frac{\sqrt{t}\bar{\bm s}_{t+1}}{\gamma}, S) \|_2^2 \\
&\min_{S\in \mathbb{M}} \|-\frac{\sqrt{t}\bar{\bm s}_{t+1}}{\gamma} - P(-\frac{\sqrt{t}\bar{\bm s}_{t+1}}{\gamma}, S) \|_2^2,
\end{align}
where $P({\bm s}, S)$ is the projection operator that projects ${\bm s}$ onto the subspace spanned by $S$.
\label{theorem:theorem1}
\end{theorem}
\begin{proof}
The original minimization problem in~(\ref{definition:dual_averaging_update}) can be equivalently expressed as
\begin{align}
{\bm w}_{t+1} &= \argmin_{{\bm w} \in {\mathcal{M}(\mathbb{M})}}  \Big\{ \langle \bar{\bm s}_{t+1}, {\bm w} \rangle + \frac{\gamma}{2\sqrt{t}} \|{\bm w} \|_2^2\Big\} \nonumber\\
&= \argmin_{{\bm w} \in {\mathcal{M}}(\mathbb{M})}  \Big\{ \frac{\sqrt{t}}{2\gamma}\|\bar{\bm s}_{t+1}\|_2^2 + \langle \bar{\bm s}_{t+1}, {\bm w} \rangle + \frac{\gamma}{2\sqrt{t}} \|{\bm w} \|_2^2\Big\} \nonumber\\
&= \argmin_{{\bm w} \in 
{\mathcal{M}}(\mathbb{M})}  \frac{\gamma}{2\sqrt{t}}\Big\| {\bm w} - \Big(-\frac{\sqrt{t}}{\gamma}\bar{\bm s}_{t+1}\Big) \Big\|_2^2 \nonumber\\
&= \argmin_{{\bm w} \in 
{\mathcal{M}}(\mathbb{M})}  \Big\| {\bm w} - \Big(-\frac{\sqrt{t}}{\gamma}\bar{\bm s}_{t+1}\Big) \Big\|_2^2,
\label{equ:transformed_problem}
\end{align} 
where the second equality follows by adding a constant to the minimization objective and ~(\ref{equ:transformed_problem}) follows by multiplying $2\sqrt{t}/\gamma$ on the third equation. Hence, (\ref{definition:dual_averaging_update}) is equivalent to the minimization of~(\ref{equ:transformed_problem}). Clearly, (\ref{equ:transformed_problem}) is essentially the projection ${\rm P}(-(\sqrt{t}{\bm s}_{t+1})/\gamma,\mathcal{M}(\mathbb{M}))$ defined in~(\ref{equ:projection_operator}). To further explore~(\ref{equ:transformed_problem}), notice that one needs to solve the following equivalent minimization problem:
\begin{equation*}
\min_{{\bm w} \in \mathcal{M}(\mathbb{M})} 
\| -\frac{\sqrt{t}\bar{\bm s}_{t+1}}{\gamma} - {\bm w} \|_2^2 \Leftrightarrow 
\min_{S \in \mathbb{M} } \| -\frac{\sqrt{t}\bar{\bm s}_{t+1}}{\gamma} - {\rm P}(-\frac{\sqrt{t}\bar{\bm s}_{t+1}}{\gamma}, S) \|_2^2.
\end{equation*}
Here, for any ${\bm x}$, ${\rm P}({\bm x}, S)$ is an orthogonal projection operator that projects ${\bm x}$ onto subspace spanned by $S$. By the projection theorem, for any ${\bm x}$, it always has the following property:
\begin{equation}
\|{\bm x} \|_2^2 - \| {\rm P}({\bm x}, S)\|_2^2 = \| {\bm x} - {\rm P}({\bm x}, S) \|_2^2.
\nonumber
\end{equation}
Replacing ${\bm x}$ by $-\sqrt{t}\bar{\bm s}_{t+1}/\gamma$ and adding minimization to both sides with respect to subspace $S$, we obtain:
\begin{align}
&\min_{S\in \mathbb{M}} \Big\{\|-\frac{\sqrt{t}\bar{\bm s}_{t+1}}{\gamma} \|_2^2 - \| {\rm P}({-\frac{\sqrt{t}\bar{\bm s}_{t+1}}{\gamma}}, S)\|_2^2 \Big\} \nonumber \\
&= \min_{S \in \mathbb{M}}\| {-\frac{\sqrt{t}\bar{\bm s}_{t+1}}{\gamma}} - {\rm P}(-\frac{\sqrt{t}\bar{\bm s}_{t+1}}{\gamma}, S) \|_2^2.\nonumber
\end{align}
By moving the minimization into the negative term, we obtain
\begin{align}
&\|-\frac{\sqrt{t}\bar{\bm s}_{t+1}}{\gamma} \|_2^2 + \max_{S\in \mathbb{M}} \| {\rm P}(-\frac{\sqrt{t}\bar{\bm s}_{t+1}}{\gamma}, S)\|_2^2 \nonumber\\
&= \min_{S\in \mathbb{M}}\| -\frac{\sqrt{t}\bar{\bm s}_{t+1}}{\gamma} - {\rm P}(-\frac{\sqrt{t}\bar{\bm s}_{t+1}}{\gamma}, S) \|_2^2.
\end{align}
We prove the theorem.
\end{proof}

The above theorem leads to a key insight that the NP-hard problem~(\ref{equ:transformed_problem}) can be solved either by maximizing $\|{\rm P}(\sqrt{t}\bar{\bm s}_{t+1}/{\gamma}, S)\|_2^2$ or by minimizing $\| \sqrt{t}\bar{\bm s}_{t+1}/{\gamma} - {\rm P}(\sqrt{t}\bar{\bm s}_{t+1}/{\gamma}, S)\|_2^2$ over $S$. Inspired by \cite{hegde2016fast,hegde2015approximation,hegde2015nearly}, instead of solving these two problems exactly, we apply two approximated algorithms provided in~\cite{hegde2015nearly} to solve the problem approximately. We present the following two assumptions:

\begin{assumption}[Head Projection~\cite{hegde2016fast}] 
Let $\mathbb{M}$ and $\mathbb{M}_H$ be the predefined subspace models. Given any ${\bm w}$, there exists a $(c_H, \mathbb{M}, \mathbb{M}_H)$ Head-Projection which is to find a subspace $H\in \mathbb{M}_{H}$ such that
\begin{equation}
\|{\rm P}({\bm w}, H) \|^2 \geq c_H \cdot 
\max_{S \in \mathbb{M}} \| {\rm P}({\bm w}, S)\|^2,
\end{equation}
where $0 < c_H \leq 1$. We denote ${\rm P}({\bm w}, H)$ as ${\rm P}({\bm w},\mathbb{M},\mathbb{M}_H)$.
\end{assumption}

\begin{assumption}[Tail Projection~\cite{hegde2016fast}]
Let $\mathbb{M}$  and $\mathbb{M}_T$ be the predefined subspace models. Given any ${\bm w}$, there exists a $(c_T, \mathbb{M}, \mathbb{M}_T)$ Tail-Projection which is to find a subspace $T\in \mathbb{M}_T$ such that
\begin{equation}
\|{\rm P} ({\bm w}, T) - {\bm w} \|^2 \leq c_T \cdot 
\min_{S \in \mathbb{M}} \| {\bm w} - {\rm P}({\bm w}, S)\|^2,
\end{equation}
where $c_T \geq 1$. We denote ${\rm P}({\bm w}, T)$ as ${\rm P}({\bm w},\mathbb{M},\mathbb{M}_T)$.
\end{assumption}

To minimize the regret $R(T,\mathcal{M}(\mathbb{M}))$, we propose the approximated algorithm, presented in Algorithm~\ref{alg:relaxed-online-graph-da} below. Initially, the primal vector ${\bm w}_0$ and dual vector $\bar{\bm s}_0$ are all set to ${\bm 0}$. At each iteration, it works as the following four steps:

~$\bullet$ \textbf{Step 1}: The learner receives a question ${\bm x}_t$ and makes a prediction based on ${\bm x}_t$ and ${\bm w}_t$. After suffering a loss $f_t({\bm w}_t, \{{\bm x}_t,y_t\})$, it computes the gradient ${\bm g}_t$ in Line 4;

~$\bullet$ \textbf{Step 2}: In Line 5, the current gradient ${\bm g}_t$ has been accumulated into $\bar{\bm s}_{t+1}$, which is ready for the next head projection\footnote{Pseudo-code of these two projections are provided in Appendix~\ref{appendix:section1} for completeness.};

~$\bullet$ \textbf{Step 3}: The head projection inputs the accumulated gradient $\bar{\bm s}_{t+1}$ and outputs the vector ${\bm b}_{t+1}$ so that $\text{supp}({\bm b}_{t+1})\in \mathbb{M}_{H}$;

~$\bullet$ \textbf{Step 4}: The next predictor ${\bm w}_{t+1}$ is then updated by using the tail projection, i.e., $\text{supp}({\bm w}_{t+1}) \in \mathbb{M}_T$. The weight $-\sqrt{t}/\gamma$ is to control the learning rate.

The algorithm repeats the above four steps until some stop condition is satisfied. The main difference between our method and the methods in~\cite{nesterov2009primal,xiao2010dual} lies in that, we, as a first attempt, use two projections (Line 6 and Line 7), to project dual vector $\bar{\bm s}_{t+1}$ and primal vector ${\bm w}_{t+1}$ onto a graph-structured subspaces $\mathbb{M}_H$ and $\mathbb{M}_T$ respectively. In dual projection step, most of the irrelevant gradient entries have been effectively set to zero values. In primal tail projection step, we make sure ${\bm w}_{t+1}$ has been projected onto $\mathbb{M}_T$ so that the constraint of interest is satisfied. 

\begin{algorithm}[H]
\caption{\textsc{GraphDA}: Online Graph Dual Averaging Algorithm }
\begin{algorithmic}[1]
 \STATE \textbf{Input}: $\gamma$, $\mathbb{M}$
 \STATE ${\bar{\bm s}}_0 = {\bm 0}, {\bm w}_0 = {\bm 0}$
 \FOR{$t=0, 1, 2, \ldots$}
 \STATE receive $\{{\bm x}_t,y_t\}$ and compute ${\bf g}_t = \nabla f_t({\bm w}_t,\{{\bm x}_t,y_t\})$
 \STATE $\bar{{\bm s}}_{t+1} = {\bar{\bm s}_t} + {\bm g}_t$
 \STATE ${\bf b}_{t+1} = {\rm P}({\bf {\bar s}}_{t+1}, \mathbb{M})$
 \STATE ${\bm w}_{t+1} = {\rm P}(- \frac{\sqrt{t}}{\gamma} {\bm b}_{t+1}, \mathbb{M})$
 \ENDFOR
\end{algorithmic}\label{alg:relaxed-online-graph-da}
\end{algorithm}

In real applications, graph data is not always available, i.e., $\mathbb{M}$ cannot be explicitly constructed by $\mathbb{G}(\mathbb{V},\mathbb{E})$, so we often have to deal with non-graph data but still with the aim to pursue structure sparsity constraint. To compensate, we provide Dual Averaging Iterative Hard Thresholding, namely \textsc{DA-IHT}, presented in Theorem~\ref{theorem:theorem-da-iht}, to handle non-graph data cases.
\begin{theorem}
Assume that the graph information is not available or the graph is a complete graph and the budget $B$ is large enough. We can define our model $\mathbb{M}$ such that it includes all possible $s$-sparse subgraphs, i.e., $\mathbb{M}=\{S:|S|\leq s\}$. Then there exists exactly head and tail projection algorithm such that
\begin{equation}
\|{\rm P}({\bm w}, H) \|_2^2 = \max_{S \in \mathbb{M}} \| {\rm P}({\bm w}, S)\|_2^2,
\label{inequ:12}
\end{equation}
and 
\begin{equation}
\|{\rm P} ({\bm w}, T) - {\bm w} \|_2^2 = \min_{S \in \mathbb{M}} \| {\bm w} - {\rm P}({\bm w}, S)\|_2^2.
\label{inequ:13}
\end{equation}
\label{theorem:theorem-da-iht}
\end{theorem}
\begin{proof}
Since the graph is a complete graph (i.e., all subgraphs are connected.) and the budget constraint $B$ is large enough, any subset $S$ that has $s$ elements belongs to $\mathbb{M}$. In this case, $\mathbb{M}$ contains all $s$-subsets, i.e., $\mathbb{M} = \{ S_i: |S_i| \leq s \}$. By sorting the magnitudes of ${\bm w}$ in a descending manner, we have
\begin{equation}
|w_{\tau_1}| \geq |w_{\tau_2}| \geq \ldots \geq |w_{\tau_s}| \geq \ldots \geq |w_{\tau_p}|. \nonumber
\end{equation}
Let $H = T = \{\tau_1,\tau_2,\ldots, \tau_s\}$. For any s-sparse set $S$, by the fact that $|w_{\tau_1}|, |w_{\tau_2}|, \ldots, |w_{\tau_s}|$ are the largest magnitude $s$ entries, we always have
\begin{equation}
\|{\rm P}({\bm w}, H) \|_2^2 \geq \| {\rm P}({\bm w}, S)\|_2^2. \nonumber
\end{equation}
At the same time, $H \in \mathbb{M}$, then 
\begin{equation}
\|{\rm P}({\bm w}, H) \|_2^2 \leq \max_{S\in \mathbb{M}}\| {\rm P}({\bm w}, S)\|_2^2. \nonumber
\end{equation}
Hence, we prove~(\ref{inequ:12}). In a similar vein, we can also prove~(\ref{inequ:13}).
\end{proof}
By Theorem~\ref{theorem:theorem-da-iht}, one can implement the two projections in Line 6 and 7 of Algorithm~\ref{alg:relaxed-online-graph-da} by sorting the magnitudes $\bar{\bm s}_{t+1}$ and $-\sqrt{t}/\gamma {\bm b}_{t+1}$ respectively, to deal with non-graph data. \textsc{DA-IHT} will be used as a baseline in our experiment to compare with the graph-based method, \textsc{GraphDA}.

\textbf{Time Complexity.\quad} At each iteration of \textsc{GraphDA}, the time complexity of two projections depends on the graph size $p$ and the number of edges $|\mathbb{E}|$. As proved in~\cite{hegde2015nearly}, two projections have the time complexity $\mathcal{O}(|\mathbb{E}|\log^3(p))$. In many real-world applications, the graphs are usually sparse, i.e., $\mathcal{O}(p)$, and then the total complexity of each iteration of \textsc{GraphDA} is $\mathcal{O}(p + p\log^3(p))$. Our method is characterized by two merits: 1) The time cost of each iteration is nearly-linear time; 2) At each iteration, it only has $\mathcal{O}(p + |\mathbb{E}|)$ memory cost, where $\mathcal{O}(p)$ stores the averaging gradient and current solution and $\mathcal{O}(|\mathbb{E}|)$ is to save the graph. For \textsc{DA-IHT}, we need to select the top $s$ largest magnitude entries at each iteration. Thus, the time cost of per-iteration is $\mathcal{O}(s p)$ with $\mathcal{O}(p)$ memory cost.

\textbf{Regret Discussion.\quad} Given any online learning algorithm, we are interested in whether the regret $R(T,\mathcal{M}(\mathbb{M}))$ is sub-linear and whether we can bound the \textit{estimation error} $\|{\bm w}_t - {\bm w}^*\|_2$. We first assume the primal vectors are $\{{\bm w}_t\}_{t=0}^T$ and the dual gradient sequences $\{{\bm g}_t\}_{t=0}^T$. We then assume that the potential solution ${\bm w}$ is always bounded in $D$, i.e., $\|{\bm w}\|_2 \leq D$ and gradients are also bounded, i.e., $\|{\bm g}_t\|_2 \leq L$. Then for any $T \geq 1$ and any ${\bm w}\in \mathcal{M}(\mathbb{M})$, the regret in~\cite{xiao2010dual} can be bounded as the following:
\begin{equation}
R(T,\mathcal{M}(\mathbb{M})) \leq 2D L\sqrt{T}.
\label{theorem1:regret}
\end{equation}
Given any optimal solution ${\bm w}^* \in \argmin_{{\bm w} \in \mathcal{M}(\mathbb{M})} \sum_{i=1}^T f_i({\bm w})$ and the solution ${\bm w}_T$, the \textit{estimation error}, i.e., $\|{\bm w}_{T+1}-{\bm w}^*\|$ is bounded as the following:
\begin{equation}
\| {\bm w}_{T+1} - {\bm w}^* \|_2^2 \leq 2\Big( D^2 + \frac{L^2}{\gamma^2} - \frac{1}{\gamma\sqrt{T}} R(T,\mathcal{M}(\mathbb{M})) \Big).
\label{theorem:estimation-error}
\end{equation}
However, the regret bound~(\ref{theorem1:regret}) and \textit{estimation error}~(\ref{theorem:estimation-error}) are under the assumption that the constraint set $\mathcal{M}(\mathbb{M})$ is convex. For \textsc{GraphDA}, an approximated algorithm, it is difficult to establish a sublinear regret bound. The reasons are two-fold: 1) Due to the non-convexity of $\mathcal{M}(\mathbb{M})$, it is possible that \textsc{GraphDA} converges to a local minimal, so the regret will potentially be non-sublinear; 2) The solution of model projection is approximated, making the regret analysis harder. Although recent work~\cite{gao2018online,hazan2017efficient} shows that it is still possible to obtain a \textit{local-regret} bound when the objective function is non-convex, it is different from our case since we assume the objective function convex subject to a non-convex constraint. We leave the theoretical regret bound analysis of \textsc{GraphDA} an open problem.

\section{Experiments}
\label{section:experiments}
To corroborate our algorithm, we conduct extensive experiments, comparing \textsc{GraphDA} with some popular baseline methods. Note \textsc{DA-IHT} derived from Theorem~\ref{theorem:theorem-da-iht} is treated as a baseline method. We aim to answer the following questions:
\begin{itemize}[leftmargin=*]
\item \textbf{\textit{Question Q1}: } Can \textsc{GraphDA} achieve better classification performance compared with baseline methods?
\label{question:02}
\item \textbf{\textit{Question Q2}: } Can \textsc{GraphDA} learn an stronger interpretative model through capturing more meaningful graph-structure features compared with baseline methods?
\label{question:01}
\vspace{-5mm}
\end{itemize}
\begin{figure}[H]
\subfigure[\textit{Graph01}]{\includegraphics[width=2.05cm]
{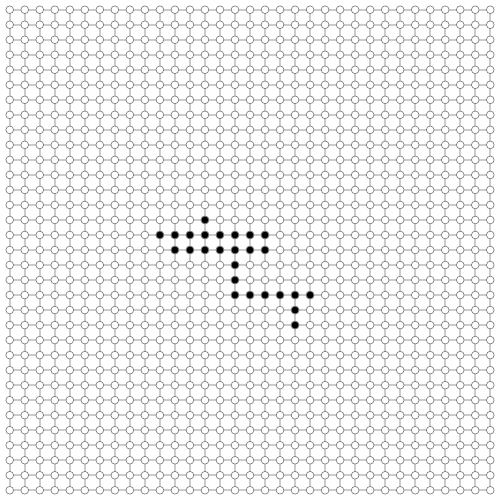}}
\hfill
\subfigure[\textit{Graph02}]{\includegraphics[width=2.05cm]
{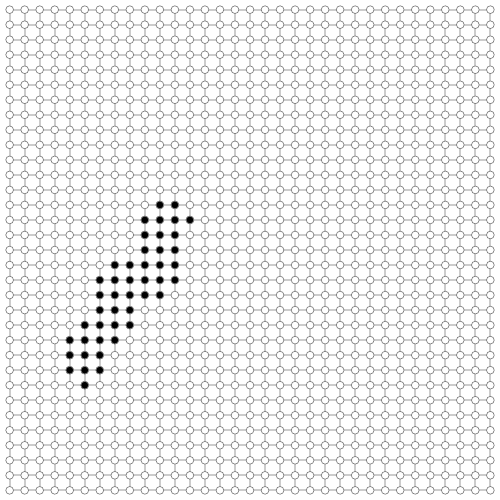}}
\hfill
\subfigure[\textit{Graph03}]{\includegraphics[width=2.05cm]
{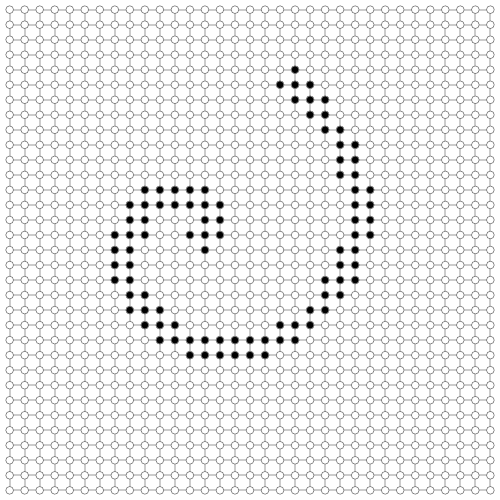}}
\hfill
\subfigure[\textit{Graph04}]{\includegraphics[width=2.05cm]
{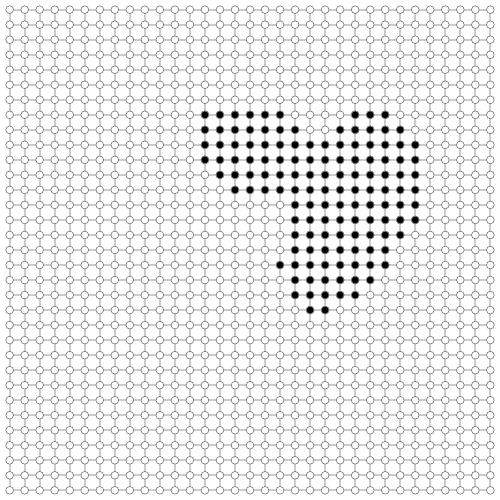}}
\vspace{-4mm}
\caption{Four benchmark graphs from~\cite{arias2011detection}}
\vspace{-2mm}
\label{figure:benchmark_graphs}
\end{figure}

\subsection{Datasets and evaluation metrics}

\textbf{Datasets.\quad} We use the following three publicly available graph datasets: 1) \textbf{Benchmark Dataset}~\cite{arias2011detection}. Four benchmark graphs~\cite{arias2011detection} are shown in Figure~\ref{figure:benchmark_graphs}. The four subgraphs are embedded into $33\times 33$ graphs with 26, 46, 92, and 132 nodes respectively. Each graph has $p=1,089$ nodes and $m = 2,112$ edges with unit weight $1.0$. We use the Benchmark dataset to learn an online graph logistic regression model; 2) \textbf{MNIST Dataset}~\cite{lecun1998mnist}. This popular hand-writing dataset is used to test \textsc{GraphDA} on online graph sparse linear regression. It contains ten classes of handwritten digits from 0 to 9. We randomly choose each digit as our target graph. Each pixel stands for a node. There exists an edge if two nodes are neighbors. We set the weights to 1.0 for edges; 3) \textbf{KEGG Pathway Dataset}~\cite{lenail2017graph}. The Kyoto Encyclopedia of Genes and Genomes (KEGG) dataset contains 5,372 genes. These genes (nodes) form a connected graph with 78,545 edges. The edge weights stand for correlations between two genes. We use KEGG to detect a related pathway.

\begin{table*}
\centering
\small
\vspace{-4mm}
\caption{Classification performance on \textit{Graph01} of Benchmark dataset}
\vspace{-4mm}
\begin{tabular}{c c c c c c c c}
\hline\hline
Method & $\textit{Pre}_{{\bm w}_t}\pm$std & $\textit{Rec}_{{\bm w}_t}\pm$std & $\textit{F1}_{{\bm w}_t}\pm$std & $\textit{AUC}_{{\bm w}_t,{\bar{\bm w}}_t}$ & $\textit{Acc}_{{\bm w}_t,{\bar{\bm w}}_t}$ & $\textit{Miss}_{{\bm w}_t,{\bar{\bm w}}_t}$ & $\textit{NR}_{{\bm w}_t,{\bar{\bm w}}_t}$\\
\hline
\textsc{Adam} & 0.024$\pm$0.00 & \textbf{1.000}$\pm$0.00 & 0.047$\pm$0.00 & (0.618, 0.603) & (0.619, 0.603) & (166.35, 173.10) & (100.0\%, 100.0\%) \\
\textsc{$\ell_1$-RDA} & 0.267$\pm$0.11 & 0.863$\pm$0.09 & 0.389$\pm$0.13 & (0.693, 0.672) & (0.694, 0.673) & (155.30, 166.05) & (11.58\%, 83.60\%) \\
\textsc{AdaGrad} & 0.256$\pm$0.11 & 0.877$\pm$0.09 & 0.379$\pm$0.13 & (0.696, 0.636) & (0.696, 0.637) & (156.00, 166.00) & (11.33\%, 100.0\%) \\
\textsc{DA-GL} & 0.176$\pm$0.11 & 0.967$\pm$0.04 & 0.283$\pm$0.12 & (0.735, 0.666) & (0.735, 0.667) & (142.90, 162.20) & (15.99\%, 100.0\%) \\
\textsc{DA-SGL} & 0.523$\pm$0.40 & 0.854$\pm$0.14 & 0.506$\pm$0.35 & (0.699, 0.647) & (0.699, 0.647) & (151.00, 165.50) & (25.54\%, 100.0\%) \\
\textsc{StoIHT} & 0.057$\pm$0.04 & 0.150$\pm$0.08 & 0.072$\pm$0.03 & (0.552, 0.523) & (0.553, 0.523) & (194.55, 195.25) & (7.79\%, 40.62\%) \\
\textsc{GraphStoIHT} & 0.151$\pm$0.12 & 0.356$\pm$0.16 & 0.194$\pm$0.12 & (0.603, 0.570) & (0.602, 0.570) & (174.65, 181.40) & (7.84\%, \textbf{22.06}\%) \\
\textsc{DA-IHT} & 0.507$\pm$0.20 & 0.744$\pm$0.12 & 0.566$\pm$0.11 & (0.697, 0.666) & (0.697, 0.666) & (155.65, 162.85) & (4.35\%, 39.50\%) \\
\textsc{GraphDA} & \textbf{0.869}$\pm$0.13 & 0.906$\pm$0.04 & \textbf{0.880}$\pm$0.08 & (\textbf{0.749}, \textbf{0.739}) 
& (\textbf{0.749}, \textbf{0.739}) & (\textbf{133.45}, \textbf{136.20}) & (\textbf{2.56}\%, 32.12\%) \\
\hline
\end{tabular}
\label{table:fig-1-selected-accuracy_graph1}
\end{table*}

\noindent\textbf{Evaluation metrics.\quad} We have two categories of metrics to answer \textbf{Question Q1} and \textbf{Question Q2} respectively. To measure classification performance of ${\bm w}_t$ or $\bar{\bm w}_{t}$\footnote{For the comparison, we also evaluate the averaged decision $\bar{\bm w}_{t}$ similar as done in~\cite{xiao2010dual}.}, we use classification Accuracy(\textit{Acc}), the Area Under Curve (\textit{AUC})~\cite{hanley1982meaning}, and the number of Misclassified samples (\textit{Miss}). To evaluate feature-level performance (interpretability), we use Precision (\textit{Pre}), Recall (\textit{Rec}), F1-score (\textit{F1}), and Nonzero Ratio (\textit{NR}). To clarify, given any optimal ${\bm w}^* \in \mathbb{R}^p$ and learned model ${\bm w}_t$, \textit{Pre}, \textit{Rec}, \textit{F1}, and \textit{NR} are defined as follows:
\begin{align}
Pre_{{\bm w}_t} &= \frac{|\text{supp}({\bm w}^*) \cap \text{supp}({\bm w}_t)|}{|\text{supp}({\bm w}_t)|} ,\ 
Rec_{{\bm w}_t} = \frac{|\text{supp}({\bm w}^*) \cap \text{supp}({\bm w}_t)|}{|\text{supp}({\bm w}^*) |}\nonumber\\
\text{\textit{F1}}_{{\bm w}_t} &= \frac{2 |\text{supp}({\bm w}^*) \cap \text{supp}({\bm w}_t) |}{ |\text{supp}({\bm w}^*) | + |\text{supp}({\bm w}_t) |},\ \text{NR}_{\bm w} = \frac{ |\text{supp}({\bm w}) |}{p}.
\label{definition:node_pre_rec_fm}
\end{align}

\subsection{Baseline methods}
We consider the following eight baseline methods: 1) \textsc{$\ell_1$-RDA}~\cite{xiao2010dual}. We use the enhanced Regularized Dual-Averaging ($\ell_1$-RDA) method in Algorithm 2 of~\cite{xiao2010dual}; 2) \textsc{DA-GL}~\cite{yang2010online}. Online Dual Averaging Group Lasso (\textsc{DA-GL}) is the dual averaging method with group Lasso; 3) \textsc{DA-SGL}~\cite{yang2010online}. It also uses dual averaging, but with sparse group Lasso as the regularization; 4) \textsc{AdaGrad}~\cite{duchi2011adaptive}. The adaptive gradient with $\ell_1$ regularization is different from \textsc{$\ell_1$-RDA}~\cite{xiao2010dual}. \textsc{AdaGrad} yields a dedicated step size for each feature inversely. In order to capture the sparsity, we use its $\ell_1$ norm-based method for comparison; 5) \textsc{Adam}~\cite{kingma2014adam}. Since there is no sparsity regularization in \textsc{Adam}, it generates totally dense models. We use its online version\footnote{One can find more details of the online version in Section 4 of~\cite{kingma2014adam}.} to compare with these sparse methods; 6) \textsc{StoIHT}~\cite{nguyen2017linear}. We use this method with block size 1, which can be treated as online learning setting; 7) \textsc{DA-IHT}, derived from Theorem~\ref{theorem:theorem-da-iht} in this paper. We use it to compare with \textsc{GraphDA}, which has graph-structure constraint; 8) \textsc{GraphStoIHT}. We apply the head and tail projection to \textsc{StoIHT} to generate \textsc{GraphStoIHT}. 

\textbf{Online Setting.} All methods are completely online, i.e., all learning algorithms receive a single sample per-iteration. Due to the space limit, the parameters of all baseline methods including \textsc{GraphDA} are in Appendix~\ref{appendix:section1}. All numerical results are averaged from 20 trials. The following three sections report and discuss the experimental results on each dataset to answer \textbf{Q1} and \textbf{Q2}.

\subsection{Results from Benchmark dataset}
Given the training dataset $\{{\bm x}_i,y_i\}_{i=1}^t$, where ${\bm x}_i \in \mathbb{R}^p$ and ${y_i}\in \{\pm 1\}$ on the fly, the online graph sparse logistic regression is to minimize the regret $R({\bm t},\mathcal{M}(\mathbb{M}))$
where $f_{t}({\bm w}_t)$ is a logistic loss defined as
\begin{equation}
f_t({\bm w}_t,\{{\bm x}_t,y_t\}) = \log ( 1+ \exp (- y_t \cdot {\bm w}_t^\top {\bm x}_t)). \nonumber
\end{equation}
We simulate the negative and positive samples as done in~\cite{arias2011detection}. $y_t=-1$ stands for no signals or events (``business-as-usual''). $y_t=+1$ means a certain event happens such as disease outbreak/computer virus hidden in current data sample ${\bm x}_t$, and feature values in subgraphs are abnormally higher. That is, if $y_t = -1$, then $x_{v_i}\sim \mathcal{N}(0,1) \ \forall \ v_i \in \mathbb{V}$; and if $y_t = +1$, then,
\begin{equation}
    (x_i)_{v_i} \sim \begin{cases} \mathcal{N}(\mu,1) & v_i \in F \\
    \mathcal{N}(0,1) & v_i \ne F,
    \end{cases}
\end{equation}
where $F$ stands for the nodes of a specific subgraph showcased in Figure~\ref{figure:benchmark_graphs}. Then each entry $({\bm w}^*)_i$ is $\mu$ if $i\in F$; otherwise 0. We first fix $\mu=0.3$ and then generate validating, training, and testing samples, each with 400 samples. All methods stop at $t=400$ after seeing all training samples once. Parameters are tuned on 400 validating samples. We test ${\bm w}_t$ and $\bar{\bm w}_t$ on testing samples.

\begin{figure}
\centering
\includegraphics[width=8.3cm,height=4.5cm]{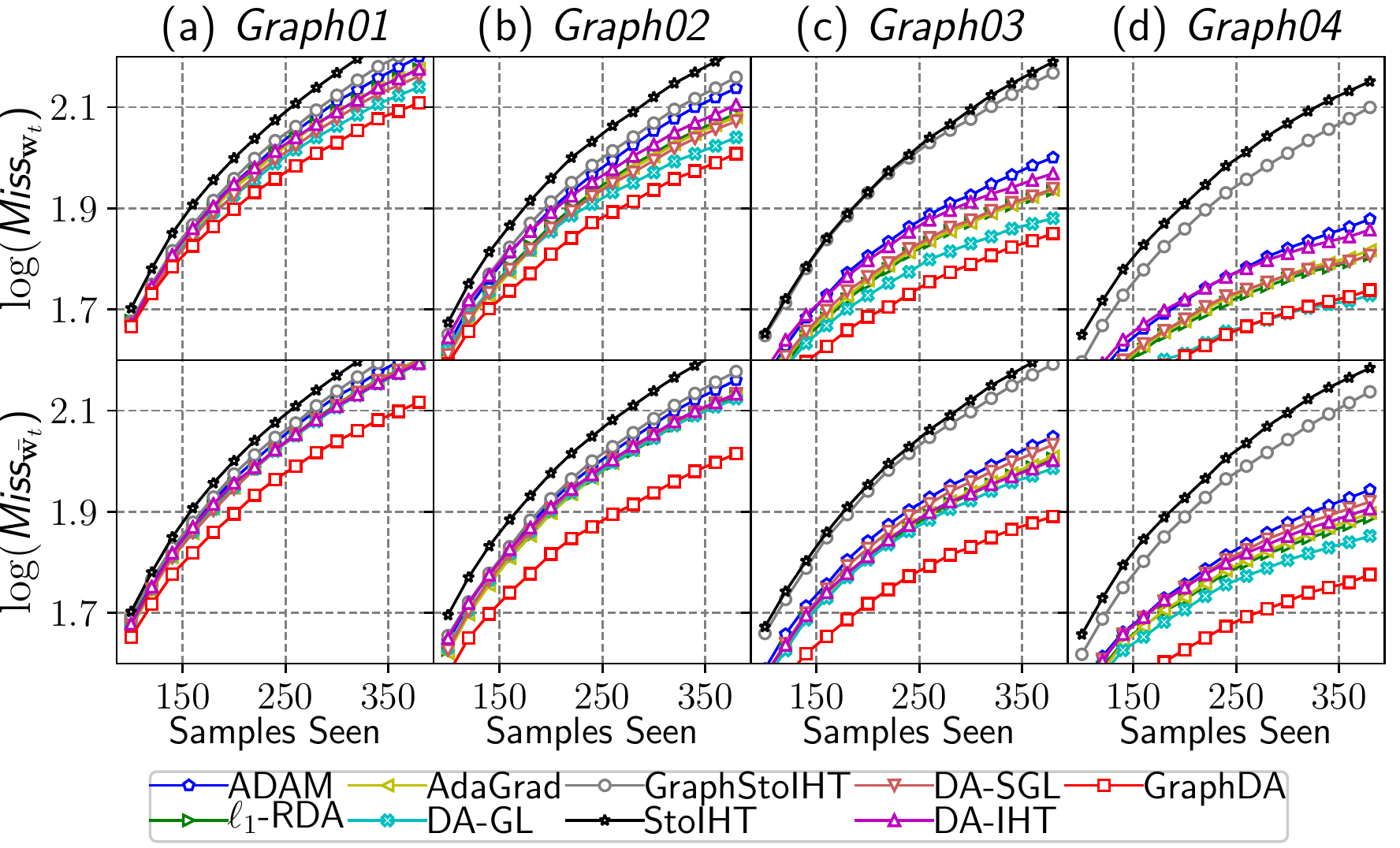}
\vspace{-2mm}
\caption{The logarithm of the number of misclassified samples as a function of samples seen}
\vspace{-4mm}
\label{figure_missed_samples_curve_0}
\end{figure}

\begin{figure*}[ht!]
    \centering
    \includegraphics[width=16.9cm,height=6.8cm]{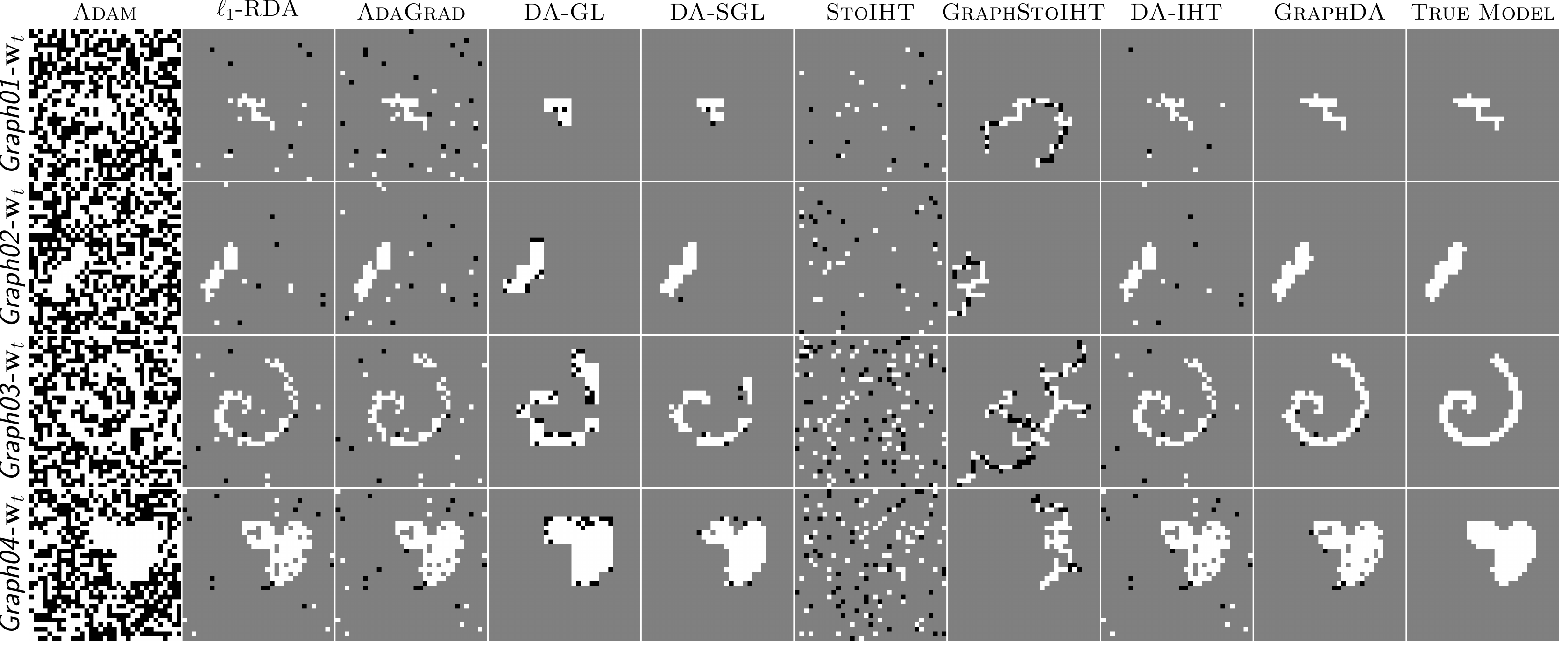}
    \vspace{-4mm}
    \caption{The learned models ${\bm w}_t$ of four benchmark graphs. These models are from the first trial of all 20 trials. For each pixel $i$, black stands for $({\bm w}_t)_i <0$, gray $({\bm w}_t)_i = 0$, and white $({\bm w}_t)_i >0$.}
    \label{figure:learned_models}
    \vspace{-4mm}
\end{figure*}

\textbf{Classification Performance on fixed $\mu$.\quad} Table~\ref{table:fig-1-selected-accuracy_graph1} shows that four all three indicators of classification performance, \textsc{GraphDA} scores higher than the other baseline methods. Specifically, it has the highest \textit{Acc} (0.749, 0.739) and \textit{AUC} (0.749, 0.739) with respect to ${\bm w}_t$ and $\bar{{\bm w}}_t$. The averaged number of misclassified samples (\textit{Miss}) is lower (133.45, 136.20), than other methods by quite a large margin. Figure~\ref{figure_missed_samples_curve_0} further shows that the number of misclassified samples of \textsc{GraphDA} keeps the lowest during the entire online learning course for all four graphs~\cite{arias2011detection}. 

\begin{figure}[H]
\centering
\vspace{-3mm}
\includegraphics[width=8.3cm,height=2.2cm]{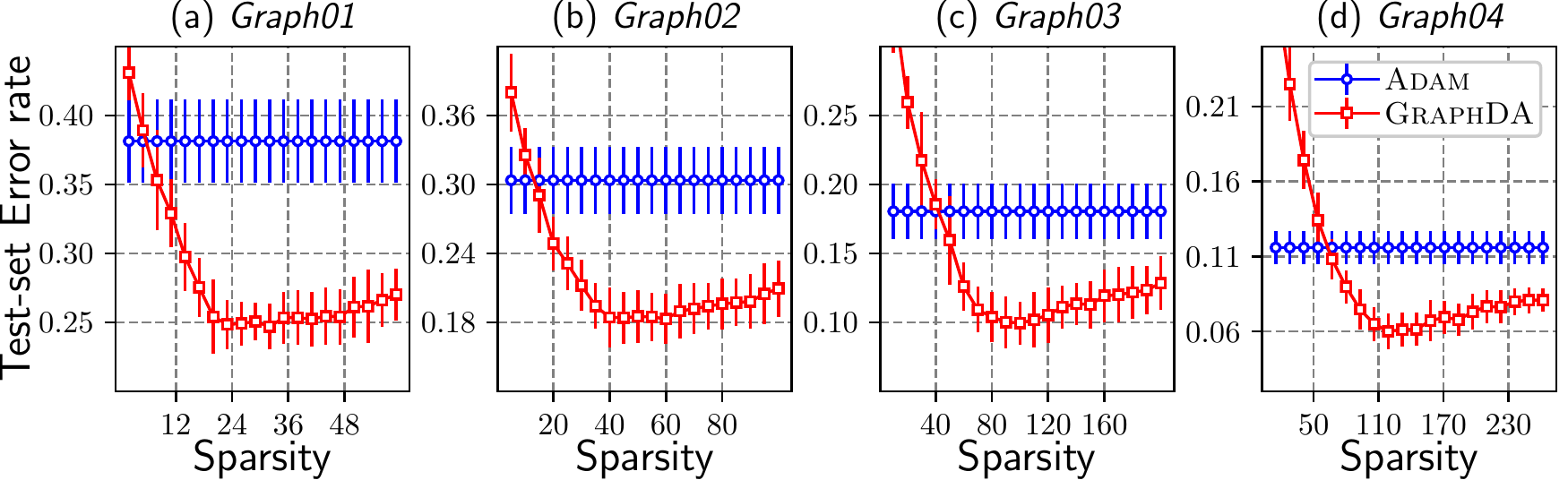}
\vspace{-3mm}
\caption{Test dataset error rates as a function of sparsity $s$}
\vspace{-5mm}
\label{figure_diff_sparsity}
\end{figure}

The sparsity $s$ is an important parameter for \textsc{GraphDA}. We explore how $s$ affects the test error rate\footnote{The test error rate is calculated as $1 - Acc_{{\bm w}_t}$ similar as done in~\cite{duchi2011adaptive}.}.  We compare the error rate of \textsc{GraphDA} with that of the non-sparse method \textsc{Adam}. Figure~\ref{figure_diff_sparsity} clearly demonstrates that \textsc{GraphDA} has the least test error rate corresponding to the true model sparsity (26, 46, 92, 132 for these four subgraphs). When $s$ reaches the true sparsity ($26, 46, 92, 132$ respectively), the testing error rate of \textsc{GraphDA} is the minimum.

\textbf{Classification Performance on different $t$ and $\mu$.\quad} We explore how different numbers of training sample and different $\mu$ affects the performance of each method similarly done in~\cite{yang2010online}. First, we choose $t$ from set $\{100, 200, 300, \ldots, 1000\}$, and tune the model based on classification accuracy. Results in Figure~\ref{figure_diff_n} show that when the number of training sample increases, the classification accuracy of all methods are increasing accordingly, but \textsc{GraphDA} enjoys the highest classification accuracy on both ${\bm w}_t$ and $\bar{\bm w}_{t}$. Second, we choose $\mu$ from the set $\{0.1, 0.2,\ldots,1.0\}$ and fix $t=400$. As is reported in Figure~\ref{figure_diff_mu}, when $\mu$ is small (a harder classification task), all methods achieve lower accuracy; when $\mu$ is large (an easier task), all methods can obtain very high accuracy except \textsc{StoIHT} and \textsc{GraphIHT}. Again, \textsc{Acc} of \textsc{GraphDA} is the highest.

\begin{figure}[H]
\centering
\vspace{-4mm}
\includegraphics[width=8cm,height=3cm]{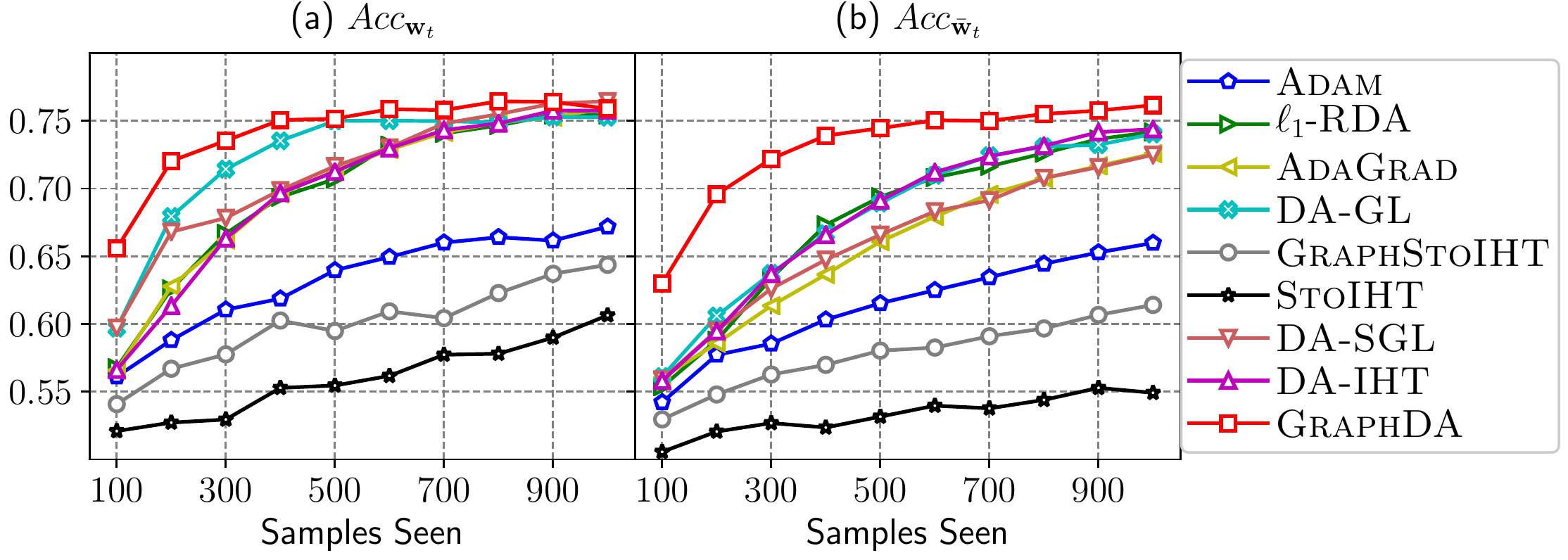}
\vspace{-4mm}
\caption{The classification accuracy on testing dataset as a function of number of training samples seen}
\vspace{-3mm}
\label{figure_diff_n}
\end{figure}

\textbf{Model interpretability.\quad} Table~\ref{table:fig-1-selected-accuracy_graph1} shows that three out of the four indicators of feature-level performance for \textsc{GraphDA} score higher than for the other baseline methods. To be more specific, our method has the highest \textit{F1}, 0.880, exceeding other methods by a large margin, which means that graph-structured information does help improve its performance. It also testifies that the head/tail projection during the online learning process does help capture more meaningful features than others. The nonzero ratio (\textit{NR}) of ${\bm w}_t$ and ${\bar{{\bm w}}_t}$ is the least. The learned ${\bm w}_t$ in Figure~\ref{figure:learned_models} shows \textsc{GraphDA} successfully captures these subgraphs in ${\bm w}_t$, which are the closest to true models in terms of shapes and values (white colored pixels). \textsc{Adam} learns a totally dense model, and hence has worse performance. \textsc{DA-IHT} and \textsc{$\ell_1$-RDA} obtain very similar performance, probably because both of them use the dual averaging techniques. The results of \textsc{StoIHT} and \textsc{GraphStoIHT} testify that the online PGD-based methods hardly learn an interpretable model. In brief, our algorithm exploits the least number of features to learn the best model among all of the methods.

\begin{figure}[ht!]
\centering
\vspace{-2mm}
\includegraphics[width=8cm,height=3cm]{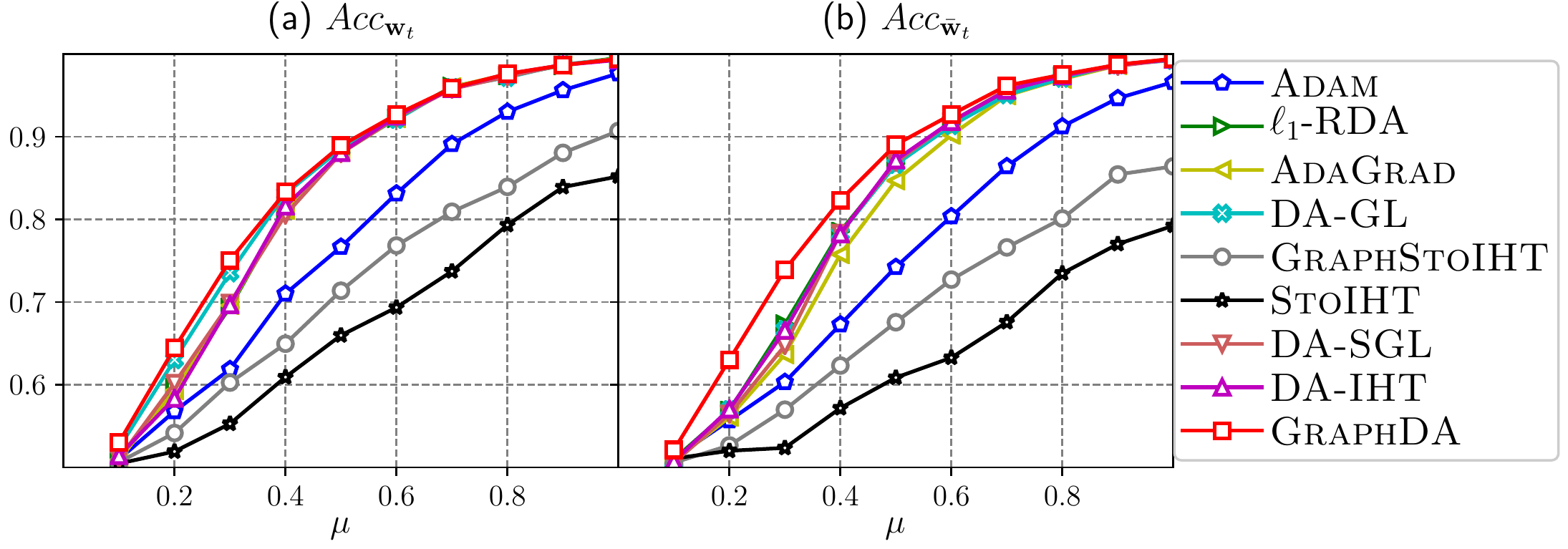}
\vspace{-3mm}
\caption{The classification accuracy on testing dataset as a function of $\mu$.}
\vspace{-5mm}
\label{figure_diff_mu}
\end{figure}

\subsection{Results from MNIST dataset}
The goal of online graph sparse linear regression is to minimize the regret where each $f_{t}({\bm w}_t)$ is the least square loss defined as
\begin{equation}
f_t({\bm w}_t,\{{\bm x}_t,y_t\}) = ( y_t - \langle {\bm w}_t,{\bm x}_t \rangle )^2.
\end{equation}
On this dataset, we use the least square loss as the objective function. The experiment is to compare the feature-level \textit{F1} score of different algorithms. We generate 1,400 data samples by using the following linear relation:
\begin{equation*}
    y_t = {\bm x}_t^\top {\bm w}^*,
\end{equation*}
where ${\bm x}_t \in \mathcal{N}({\bm 0},{\bm I})$. We use three different strategies to obtain ${\bm w}^*$. The first one is to directly use the sparse images and then normalize them to the range $[0.0,1.0]$, which we call \textit{Normalized Model}. The second is to generate ${\bm w}^*$ by letting all non-zeros be $1.0$, which is called \textit{Constant Model}. The third is to generate the nonzero nodes by using Gaussian distribution $({\bm w}^*)_i \sim \mathcal{N}(0,1)$ independently, which is \textit{Gaussian Model}. Again, our dataset is partitioned into three parts: training, validating and testing samples. We increase the number of training samples $n$  from $\{50,100,\ldots,1000\}$ and then use 200 samples as validating dataset to tune the model.  For all the eight online learning algorithms, we pass each training sample once and stop training when all training samples are used. The results shown in Figure~\ref{figure:mnist-0-4-5} are generated from the 200 testing samples.

\begin{figure}
\centering
\includegraphics[width=7.5cm,height=8cm]{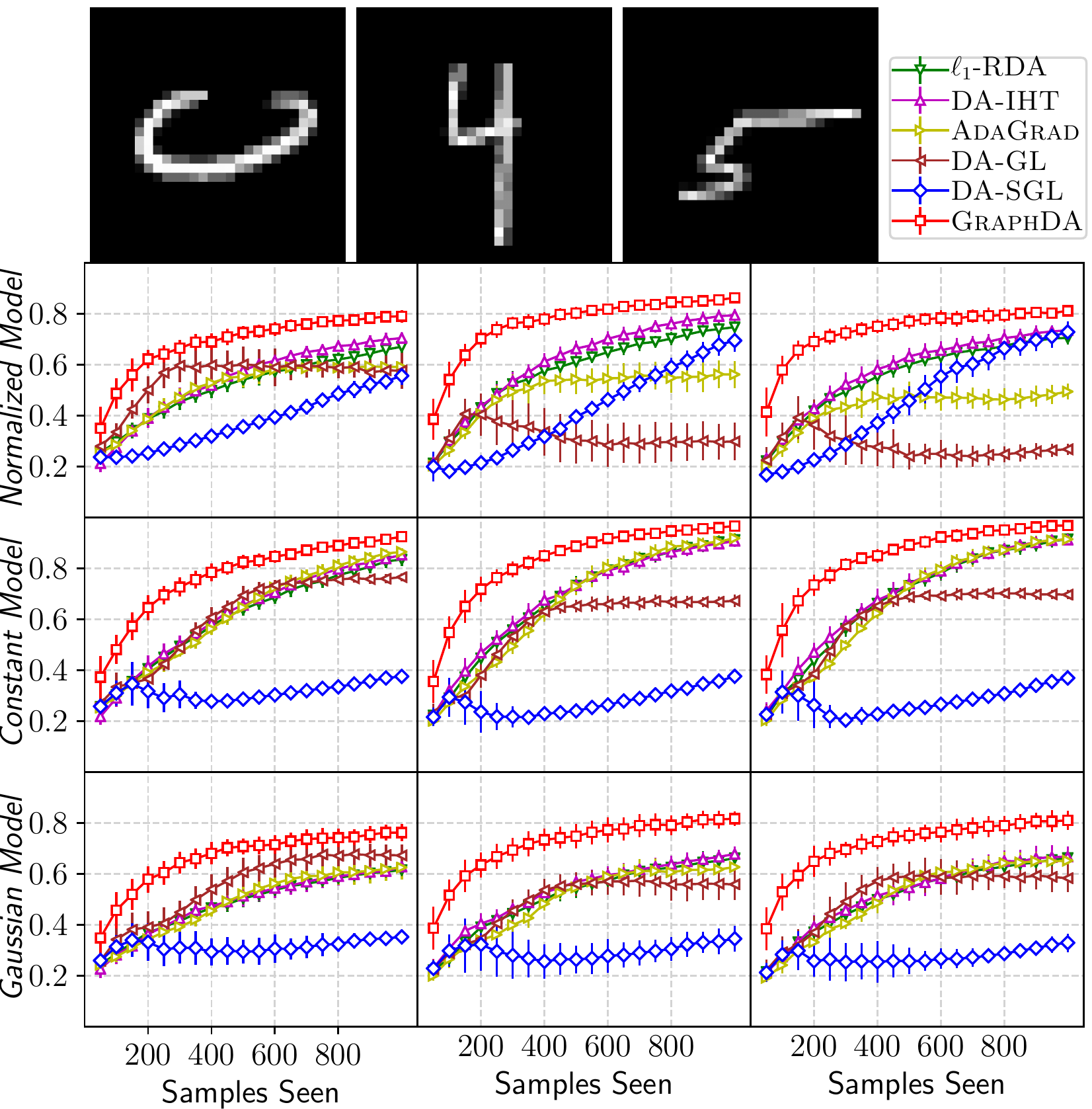}
\vspace{-3mm}
\caption{Three handwritten digits 0, 4 and 5 (top row) and the \textit{F1} score as a function of samples seen (2nd to 4th row).}
\vspace{-6mm}
\label{figure:mnist-0-4-5}
\end{figure}

\textsc{Adam}, \textsc{StoIHT} and \textsc{GraphStoIHT} are excluded from comparison because of their inferior performance. From Figure~\ref{figure:mnist-0-4-5}, we can observe that when the training samples increase, the \textit{F1} score of all methods is increasing correspondingly. But the \textit{F1} score values of \textsc{GraphDA} in \textit{Normalized Model}, \textit{Constant Model}, and \textit{Gaussian Model} are the highest among the six methods.

\subsection{Results from KEGG dataset}
\label{section:experiments:kegg}

To demonstrate that \textsc{GraphDA} can capture more meaningful features during online learning process, we test it on a real-world protein-protein iteration (PPI) network in~\cite{lenail2017graph}\footnote{It was originally provided in KEGG~\cite{kanehisa2016kegg}}. This online learning scenario could be realistic since the training samples can be collected on the fly. More Details of the dataset including the data preprocessing are in Appendix~\ref{appendix:section3}. We explore a specific gene pathway, HSA05213, related with endometrial cancer\footnote{Details of pathway HSA05213(50 genes) can be found in \url{https://www.genome.jp/dbget-bin/www_bget?hsa05213}}. Due to the lack of true labels and ground truth features (genes), we directly use the two data generation strategies in~\cite{lenail2017graph}, namely \textit{Strategy 1} (corresponding to a hard case) and \textit{Strategy 2} (corresponding to an easy case). After the data generation, we have 50 ground truth features. The number of positive and negative samples are both 100. The goal is to find out how many endometrial cancer-related genes are learned from different algorithms as done by~\cite{lenail2017graph}. All algorithms stop when they have seen 200 training samples. 

\begin{figure}[ht!]
\centering
\vspace{-3mm}
\includegraphics[width=8.2cm,height=3.2cm]{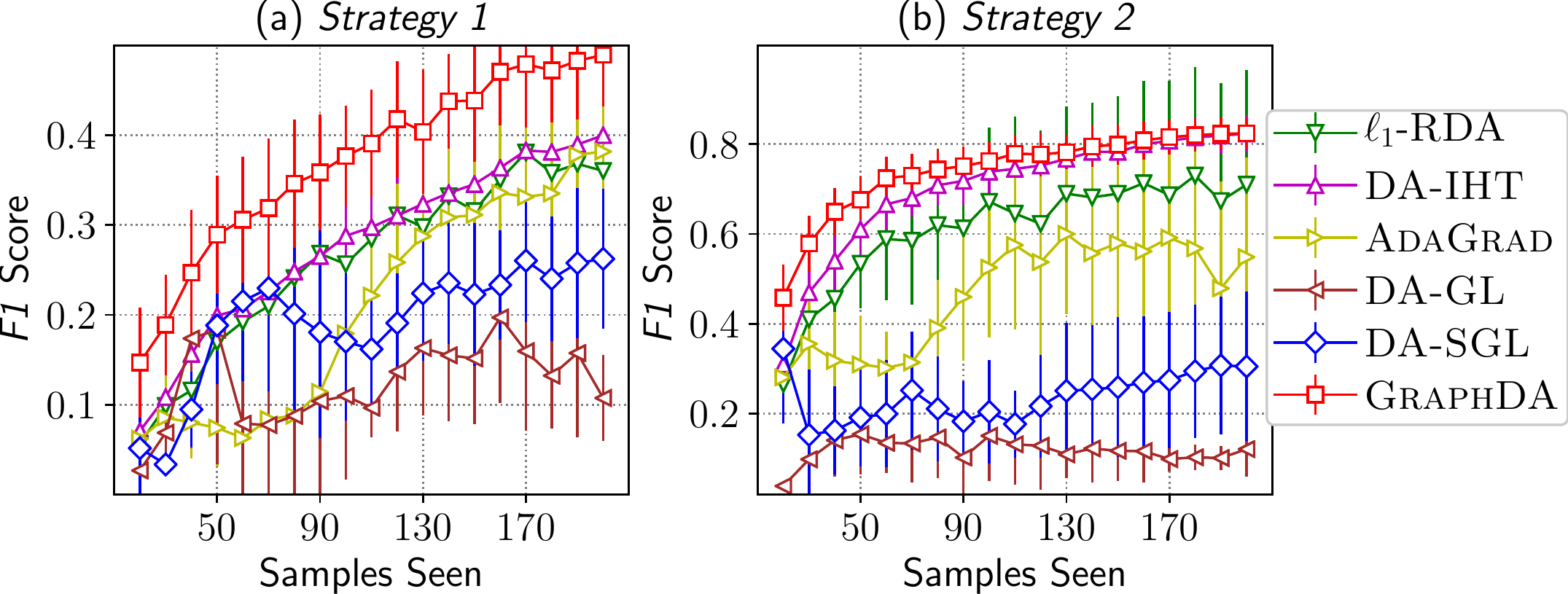}
\vspace{-3mm}
\caption{Node level \textit{F1} score as a function of the number of training samples have seen.}
\vspace{-3mm}
\label{figure:kegg-f1-score}
\end{figure}

We report the feature \textit{F1} score in Figure~\ref{figure:kegg-f1-score}. \textsc{GraphDA} outperforms the other baseline methods in terms of both two strategies, with \textit{F1} score about 0.5 for \textit{Strategy 1} and about 0.9 for \textit{Strategy 2}, higher than the rest methods. Interestingly, \text{DA-OL} and \text{DA-SOL} achieve better results only between 60 and 70 training samples and then become worse between 70 and 100. A possible explanation is that the learned model selected by the tuned parameters is not steady when the number of training samples seen is small. In addition to a better \textit{F1} score, another strength of \textsc{GraphDA} and \textsc{DA-IHT} is that the standard deviation of \textit{F1} score is smaller than other convex-based methods, including \textsc{$\ell_1$-RDA}, \textsc{DA-GL}, and \textsc{DA-SGL}.

\begin{figure}[ht!]
\centering
\vspace{-2mm}
\includegraphics[width=8cm,height=5cm]{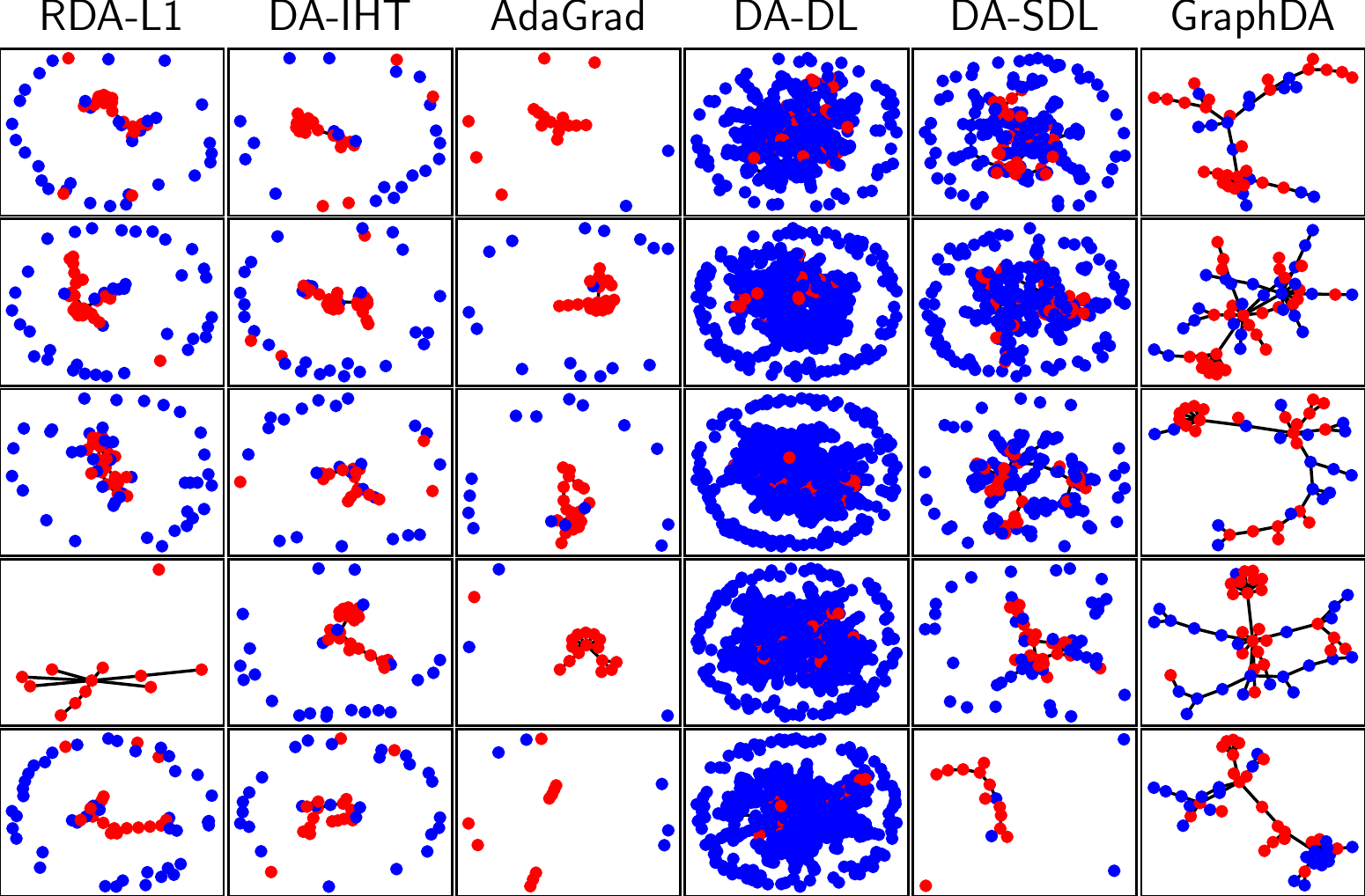}
\vspace{-2mm}
\caption{HSA05213 pathway detected by different methods. The red nodes are the genes in HSA05213 while blue nodes are the genes not in HSA05213. Results of each row are from a specific trial (from trial 1 to trial 5). Each column shows the results found by a specific method.}
\vspace{-3mm}
\label{figure:kegg-learned-genes}
\end{figure}

In Figure~\ref{figure:kegg-learned-genes}, we show the identified genes by different methods. Clearly, \textsc{GraphDA} can find more meaningful genes, indicated by less blue nodes (genes found but not in HSA05213) and more red nodes (genes found and in HSA05213). However, all the other five baseline methods have many isolated nodes (not connected to cancer-related genes).

\section{Conclusion and Future Work}
\label{section:conclusion}
In this paper, we propose a dual averaging-based method, \textsc{GraphDA}, for online graph-structured sparsity constraint problems. We prove that the minimization problem in the dual averaging step can be  formulated as two equivalent optimization problems. By projecting the dual vector and primal variables onto lower dimensional subspaces, \textsc{GraphDA} can capture graph-structure information more effectively. Experimental evaluation on one benchmark dataset and two real-world graph datasets shows that \textsc{GraphDA} achieves better classification performance and stronger interpretability compared with the baseline methods so as to answer the two questions raised at the beginning of the experiment section. It remains interesting if one can prove that both the exact and approximated projections have non-regret bound under some proper assumption, and if one can explore learning a model under the setting that true features are time evolving~\cite{hou2017learning}. 

\section{Acknowledgements}
The work of Yiming Ying is supported by the National Science Foundation (NSF) under Grant No \#1816227. The work of Baojian Zhou and Feng Chen is supported by the NSF under Grant No \#1815696 and \#1750911.

\bibliographystyle{ACM-Reference-Format}
\bibliography{reference}

\clearpage
\appendix
\section{Reproducibility}
\label{appendix:section1}
\subsection{Implementation details}
All experiments are tested on a server of Intel Xeon(R) 2.40GHZ E5-2680 with 251GB of RAM. The code is written in Python2.7 and C language with the standard C11. The implementation of the head and tail projection follows the original implementation\footnote{The two projections were originally implemented in C++, which are available at: \url{https://github.com/ludwigschmidt/cluster_approx}. We implement them by using C language. Taking the advantage of the continuous memory of arrays in C, our code is faster than original one.}. We present the pseudo code in Algorithm~\ref{alg:head_projection} below. The two projections are essentially two binary search algorithms. Each iteration of the binary search executes the Prize Collecting Steiner Tree (PCST) algorithm~\cite{johnson2000prize} on the target graph. Both projections have two main parameters: a lower bound sparsity $s_l$ and an upper bound  sparsity $s_h$. In all of the experiments, two sparsity parameters have been set to $s_l = p/2$ and $s_h=s_l*(1 + \omega)$ for the head projection, where $\omega$ is the tolerance parameter set to $0.1$. For the tail projection, we set $s_l = s$ and $s_h = s_l*(1 + \omega)$. The binary search algorithm terminates when it reaches $max\_iter=20$ maximum iterations. Line 7 of Algorithm~\ref{alg:head_projection} is the PCST algorithm proposed in~\cite{hegde2014fast_pcst}. We use a non-root version and Goemans-Williamson pruning method to prune the final forest.

\begin{algorithm}[H]
\caption{Head/Tail Projection (${\rm P}({\bm w}, \mathbb{M})$)~\cite{hegde2015nearly}}
\begin{algorithmic}[1]
 \STATE \textbf{Input}: ${\bm w}, \textit{max\_iter}, \mathbb{M} = (\mathbb{G}(\mathbb{V},\mathbb{E},{\bm c}), s_l, s_h, g)$
 \STATE ${\bm \pi} = {\bm w} \cdot {\bm w}$  \quad// vector dot product, i.e., $\pi_i = w_i * w_i$
 \STATE $\lambda_l = 0, \lambda_h = \max \{\pi_1,\pi_2,\ldots,\pi_p\}$, $\lambda_m = 0$, $t =0$
 \REPEAT 
 \STATE $\lambda_m = (\lambda_l + \lambda_h) / 2$ 
 \STATE ${\bm c}_m = \lambda_m \cdot {\bm c}$  \quad// scale dot product, i.e., $({\bm c}_m)_{i} = \lambda_m * c_i$
 \STATE $\mathcal{F}$ = PCST($\mathbb{G}(\mathbb{V},\mathbb{E},{\bm c}_m),{\bm \pi}, g$)
 \IF{ $s_l < |\mathcal{F}| < s_h$ }
 \RETURN ${\bm w}_\mathcal{F}$
 \ENDIF
 \IF{ $|\mathcal{F}| > s_h$ }
 \STATE $\lambda_l = \lambda_m$
 \ELSE
 \STATE $\lambda_h = \lambda_m$
 \ENDIF
 \STATE $t = t + 1$
 \UNTIL{t > max\_iter}
 \STATE ${\bm c}_h = \lambda_h \cdot {\bm c}$
 \STATE $\mathcal{F}$ = PCST($\mathbb{G}(\mathbb{V},\mathbb{E},{\bm c}_h),{\bm \pi}, g$)
 \RETURN ${\bm w}_\mathcal{F}$
\end{algorithmic}\label{alg:head_projection}
\end{algorithm}

\subsection{Parameter tuning}
Initial parameters ${\bf w}_0$ of all baseline methods and proposed algorithms are zero vectors ${\bf w}_0={\bf 0}$, which means we train all methods starting from a zero point. We list all related methods and their corresponding parameter settings below. (1) \textsc{$\ell_1$-RDA} is the enhanced version provided in Algorithm 2 of~\cite{xiao2010dual}. There are three parameters: The $\ell_1$-regularization parameter $\lambda$ is chosen from $\{$0.0001, 0.0005, 0.001, 0.005, 0.01, 0.03, 0.05, 0.1, 0.3, 0.5, 1.0, 3.0, 5.0, 10.0$\}$ which is a superset used in~\cite{xiao2010dual}. The parameter $\gamma$ to control the learning rate is chosen from $\{$1.0, 5.0, 10.0, 50.0, 100.0, 500.0, 1000.0, 5000.0, 10000.0 $\}$, and the sparsity-enhancing parameter $\rho$ is chosen from $\{$0.0, 0.00001, 0.000005, 0.0001, 0.0005, 0.001, 0.005, 0.01, 0.05, 0.1, 0.5, 1.0$\}$, where 0.0 is for the basic regularization. All the three parameter sets are supersets used in~\cite{xiao2010dual}. (2) \textsc{Adam}. We directly use the parameters $\beta_1 = 0.9, \beta_2 = 0.999, \epsilon=10^{-8}$ provided in~\cite{kingma2014adam}. For the magnitude of steps in parameter space $\alpha$, we choose it from $\{$0.0001, 0.0005, 0.001, 0.005, 0.01, 0.05, 0.1, 0.5 $\}$. (3) \textsc{DA-GL/SGL} have two main parameters, $\lambda$ to control the sparsity and $\gamma$ to control the learning rate. We choose $3\times 3$ grids as groups for Benchmark dataset and choose $2\times 2$ grids for MNIST dataset. (4) \textsc{DA-SGL} has an additional parameter $\gamma_g$, which is set to 1.0 for all groups as done in~\cite{yang2010online}. For each group $i$, there exists an additional parameter $r_i$ for DA-SGL. We set it as default value $r_i=1$ as recommended by the authors. (5) \textsc{AdaGrad} has two main parameters, $\lambda$ to control sparsity and $\eta$ to control the learning rate, which is from $\{$0.0001 , 0.0005, 0.001, 0.005, 0.01, 0.05, 0.1, 0.5, 1.0, 5.0, 10.0, 50.0, 100.0, 500.0, 1000.0, 5000.0$\}$. (6) \textsc{StoIHT} has two parameters: sparsity $s$ from $\{$5, 10, 15, 20, 25, 26, 30, 35, 40, 45, 46, 50, 55, 60, 65, 70, 75, 80, 85, 90, 92, 95, 100, 105, 110, 115, 120, 125, 130, 132, 135, 140, 145, 150$ \}$, and $\gamma$ to control the learning rate. (7) \textsc{GraphStoIHT} shares the same parameter settings (sparsity $s$ and $\gamma$) as \textsc{GraphDA}. The block size of \textsc{GraphStoIHT} and \textsc{StoIHT} are set to 1. (8) \textsc{GraphDA} has parameters $\gamma$ and $s$.

\section{More experimental results}

\subsection{More results from Benchmark dataset}
We present the results of \textit{Graph02}, \textit{Graph03} and \textit{Graph04} in Table~\ref{table:fig-1-selected-accuracy_graph2}, \ref{table:fig-1-selected-accuracy_graph3}, \ref{table:fig-1-selected-accuracy_graph4}, respectively. Basically, we show the classification performance (\textit{Acc}, \textit{Miss}, \textit{AUC}) and feature-level performance (\textit{Pre}, \textit{Rec}, \textit{F1}, \textit{NR}). The size of validating and testing dataset are both 400. All results are averaged from 20 trials of experiment.

\setlength{\tabcolsep}{1pt}
\begin{table}[ht!]
\vspace{-1mm}
\centering
\scriptsize
\caption{Performance of \textit{Graph02}}
\begin{tabular}{c c c c c c c c}
\hline\hline
Method & $\textit{Pre}_{{\bm w}_t}$ & $\textit{Rec}_{{\bm w}_t}$ & $\textit{F1}_{{\bm w}_t}$ & $\textit{AUC}_{{\bf w}_t,{\bar{\bf w}}_t}$ & $\textit{Acc}_{{\bf w}_t,{\bar{\bf w}}_t}$ & $\textit{Miss}_{{\bf w}_t,{\bar{\bf w}}_t}$ & $\textit{NR}_{{\bf w}_t,{\bar{\bf w}}_t}$\\
\hline
\textsc{Adam} & 0.042 & \textbf{1.000} & 0.081 & (0.697, 0.663) & (0.696, 0.663) & (144, 151) & (100.0\%, 100.0\%) \\
\textsc{$\ell_1$-RDA} & 0.371 & 0.876 & 0.494 & (0.772, 0.732) & (0.772, 0.731) & (127, 140) & (13.31\%, 96.47\%) \\
\textsc{AdaGrad} & 0.342 & 0.888 & 0.470 & (0.771, 0.711) & (0.771, 0.711) & (125, 141) & (14.43\%, 100.0\%) \\
\textsc{DA-GL} & 0.270 & 0.976 & 0.415 & (0.809, 0.755) & (0.809, 0.755) & (114, 138) & (17.07\%, 100.0\%) \\
\textsc{DA-SGL} & 0.283 & 0.948 & 0.314 & (0.777, 0.738) & (0.777, 0.737) & (123, 141) & (45.42\%, 100.0\%) \\
\textsc{StoIHT} & 0.102 & 0.217 & 0.132 & (0.586, 0.557) & (0.586, 0.557) & (171, 179) & (9.48\%, 45.60\%) \\
\textsc{GraphStoIHT} & 0.279 & 0.355 & 0.287 & (0.669, 0.620) & (0.669, 0.620) & (150, 158) & (7.31\%, \textbf{19.29}\%) \\
\textsc{DA-IHT} & 0.679 & 0.741 & 0.694 & (0.776, 0.733) & (0.776, 0.733) & (132, 141) & (4.86\%, 42.86\%) \\
\textsc{GraphDA} & \textbf{0.855} & 0.870 & \textbf{0.850} & (\textbf{0.811}, \textbf{0.799}) & (\textbf{0.811}, \textbf{0.799}) & (\textbf{106}, \textbf{107}) & (\textbf{4.55}\%, 43.89\%) \\
\hline
\end{tabular}
\vspace{-1mm}
\label{table:fig-1-selected-accuracy_graph2}
\end{table}

\setlength{\tabcolsep}{1pt}
\vspace{-1mm}
\begin{table}[ht!]
\centering
\scriptsize
\caption{Performance of \textit{Graph03}}
\begin{tabular}{c c c c c c c c}
\hline\hline
Method & $\textit{Pre}_{{\bm w}_t}$ & $\textit{Rec}_{{\bm w}_t}$ & $\textit{F1}_{{\bm w}_t}$ & $\textit{AUC}_{{\bf w}_t,{\bar{\bf w}}_t}$ & $\textit{Acc}_{{\bf w}_t,{\bar{\bf w}}_t}$ & $\textit{Miss}_{{\bf w}_t, {\bar{\bf w}}_t}$ & $\textit{NR}_{{\bf w}_t,{\bar{\bf w}}_t}$\\
\hline
\textsc{Adam} & 0.084 & \textbf{1.000} & 0.156 & (0.820, 0.789) & (0.820, 0.788) & (104, 116) & (100.0\%, 100.0\%) \\
\textsc{$\ell_1$-RDA} & 0.340 & 0.940 & 0.488 & (0.870, 0.833) & (0.869, 0.833) & ( 88, 104) & (26.15\%, 99.51\%) \\
\textsc{AdaGrad} & 0.318 & 0.942 & 0.462 & (0.872, 0.825) & (0.872, 0.824) & ( 88, 106) & (29.21\%, 100.0\%) \\
\textsc{DA-GL} & 0.289 & 0.990 & 0.443 & (0.894, 0.853) & (0.894, 0.853) & ( 78, 100) & (32.07\%, 100.0\%) \\
\textsc{DA-SGL} & 0.166 & 0.990 & 0.283 & (0.883, 0.829) & (0.883, 0.828) & ( 89, 111) & (52.19\%, 100.0\%) \\
\textsc{StoIHT} & 0.156 & 0.208 & 0.175 & (0.635, 0.593) & (0.634, 0.592) & (162, 173) & (11.62\%, 50.60\%) \\
\textsc{GraphStoIHT} & 0.276 & 0.223 & 0.217 & (0.666, 0.640) & (0.667, 0.640) & (154, 163) & (8.93\%, \textbf{23.20}\%) \\
\textsc{DA-IHT} & 0.716 & 0.782 & 0.734 & (0.865, 0.834) & (0.865, 0.834) & ( 95, 103) & (9.59\%, 63.20\%) \\
\textsc{GraphDA} & \textbf{0.856} & 0.881 & \textbf{0.864} & (\textbf{0.898}, \textbf{0.885}) & (\textbf{0.897}, \textbf{0.885}) & (\textbf{72}, \textbf{80}) & (\textbf{8.82}\%, 63.14\%) \\
\hline
\end{tabular}
\vspace{-1mm}
\label{table:fig-1-selected-accuracy_graph3}
\end{table}

\setlength{\tabcolsep}{1pt}
\vspace{-1mm}
\begin{table}[ht!]
\centering
\scriptsize
\caption{Performance of \textit{Graph04}}
\begin{tabular}{c c c c c c c c}
\hline\hline
Method & $\textit{Pre}_{{\bm w}_t}$ & $\textit{Rec}_{{\bm w}_t}$ & $\textit{F1}_{{\bm w}_t}$ & $\textit{AUC}_{{\bf w}_t,{\bar{\bf w}}_t}$ & ${Acc}_{{\bf w}_t,{\bar{\bf w}}_t}$ & $\textit{Miss}_{{\bf w}_t,{\bar{\bf w}}_t}$ & $\textit{NR}_{{\bf w}_t,{\bar{\bf w}}_t}$\\
\hline
\textsc{Adam} & 0.121 & \textbf{1.000} & 0.216 & (0.884, 0.858) & (0.884, 0.858) & ( 77,  90) & (100.0\%, 100.0\%) \\
\textsc{$\ell_1$-RDA} & 0.361 & 0.961 & 0.513 & (0.917, 0.896) & (0.917, 0.896) & ( 66,  79) & (36.19\%, 99.21\%) \\
\textsc{AdaGrad} & 0.376 & 0.961 & 0.528 & (0.919, 0.889) & (0.919, 0.889) & ( 67,  81) & (35.43\%, 100.0\%) \\
\textsc{DA-GL} & 0.476 & 0.994 & 0.640 & (\textbf{0.942}, 0.918) & (\textbf{0.941}, 0.918) & ( \textbf{54},  73) & (26.03\%, 100.0\%) \\
\textsc{DA-SGL} & 0.238 & 0.988 & 0.379 & (0.931, 0.894) & (0.931, 0.894) & ( 65,  85) & (53.31\%, 100.0\%) \\
\textsc{StoIHT} & 0.207 & 0.203 & 0.204 & (0.689, 0.639) & (0.689, 0.639) & (148, 160) & (11.86\%, 47.88\%) \\
\textsc{GraphStoIHT} & 0.439 & 0.245 & 0.299 & (0.743, 0.699) & (0.743, 0.699) & (131, 143) & (\textbf{7.77}\%, \textbf{19.96}\%) \\
\textsc{DA-IHT} & 0.780 & 0.801 & 0.788 & (0.919, 0.898) & (0.919, 0.899) & ( 74,  82) & (12.51\%, 72.72\%) \\
\textsc{GraphDA} & \textbf{0.931} & 0.865 & \textbf{0.895} & (0.939, \textbf{0.925}) & (0.939, \textbf{0.925}) & (56, \textbf{61}) & (11.30\%, 72.80\%) \\
\hline
\end{tabular}
\vspace{-1mm}
\label{table:fig-1-selected-accuracy_graph4}
\end{table}

\begin{figure*}
\centering
\includegraphics[width=16cm,height=8cm]{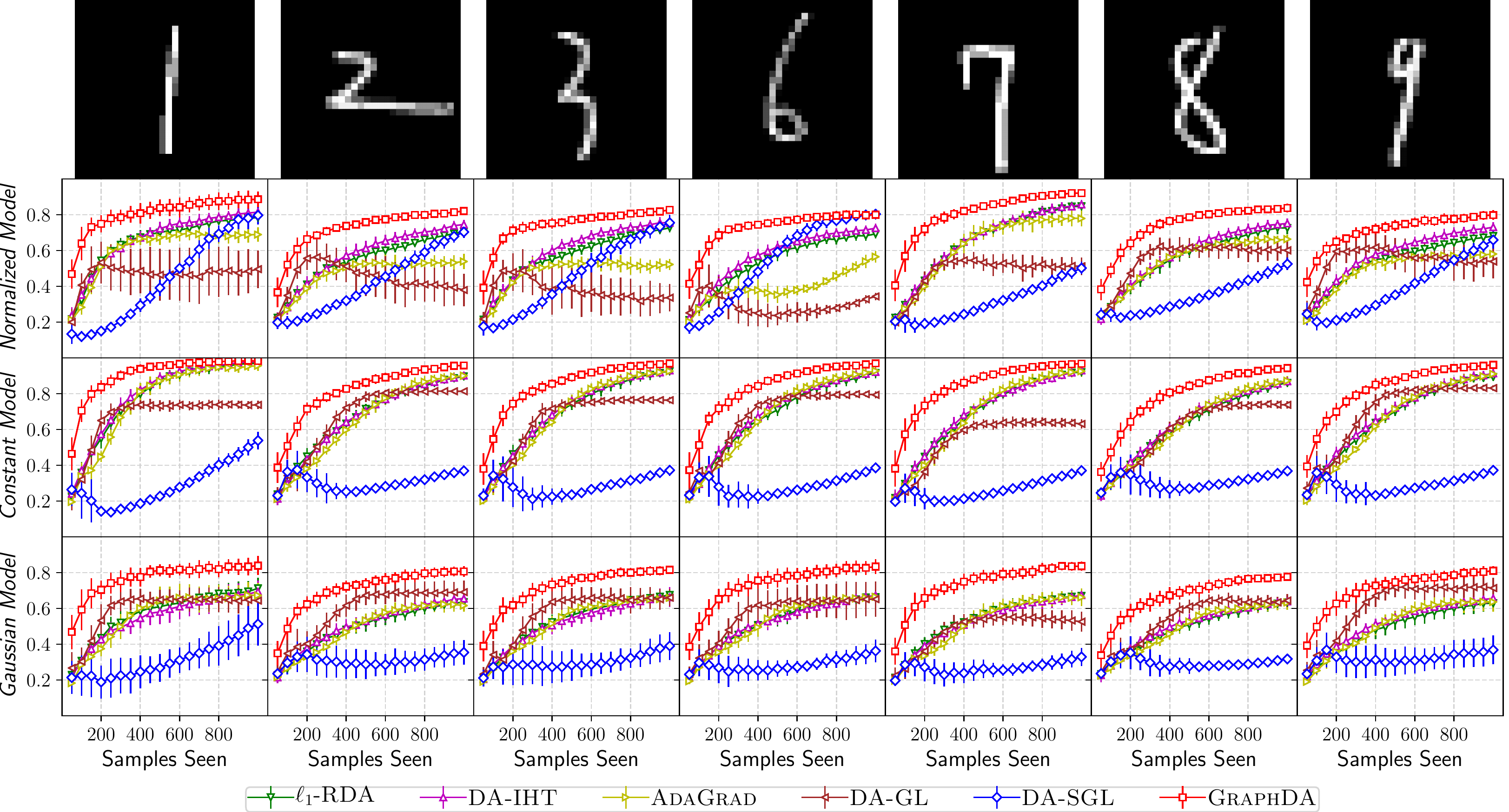}
\caption{Seven handwritten digits 1, 2, 3, 6, 7, 8 and 9 (top row) and the \textit{F1} score as a function of samples seen (2nd to 4th row)}
\label{figure:mnist_rest_results}
\end{figure*}

\subsection{More results from MNIST Dataset}
We show the results on image id $\{1, 2, 3, 6, 7, 8, 9\}$. To make the task more challenging, these 7 images are the sparsest images (with the digits forming a connected component) selected from MNIST dataset. The sparsity parameter $s$ of \textsc{DA-IHT} and \textsc{GraphDA} is chosen from $\{$30, 32, \ldots, 100$\}$ with step size 2. Figure~\ref{figure:mnist_rest_results} reports the results.

\subsection{More results from KEGG Pathways}
\label{appendix:section3}

This PPI network contains a total of 229 pathways. Each pathway often involves a specific biological function, e.g. metabolism. We restrict our analysis on 225 pathways (by removing 4 empty pathways), which contains 5,374 genes with 78,545 edges. These genes form a connected graph. There exists an edge if two proteins (genes) physically interact with each other~\cite{li2017scored}. Weights of edges stand for the confidence of these interactions. There are 7,368 genes with null values. We sample these null values from $\mathcal{N}(0,1)$. 

Due to the inferior performance of \textsc{Adam}, \textsc{StoIHT} and \textsc{GraphStoIHT}, we exclude them from experimental evaluation. Notice that \textsc{DA-GL} and \textsc{DA-SGL} need groups as priors. However, the groups (pathways) of this PPI network have overlapping features. To remedy this issue, we simply replicate these overlapping features as suggested in~\cite{yang2010online,jacob2009group}, and by doing so, the two baselines are still applicable. For these two non-convex methods, we choose the sparsity parameter from $\{$40, 45, 50, 55, 60 $\}$. We report the averaged results from 20 trials in  Figure~\ref{figure:kegg-learned-genes-rest}.

\begin{figure}[H]
\centering
\includegraphics[width=8.3cm,height=6cm]{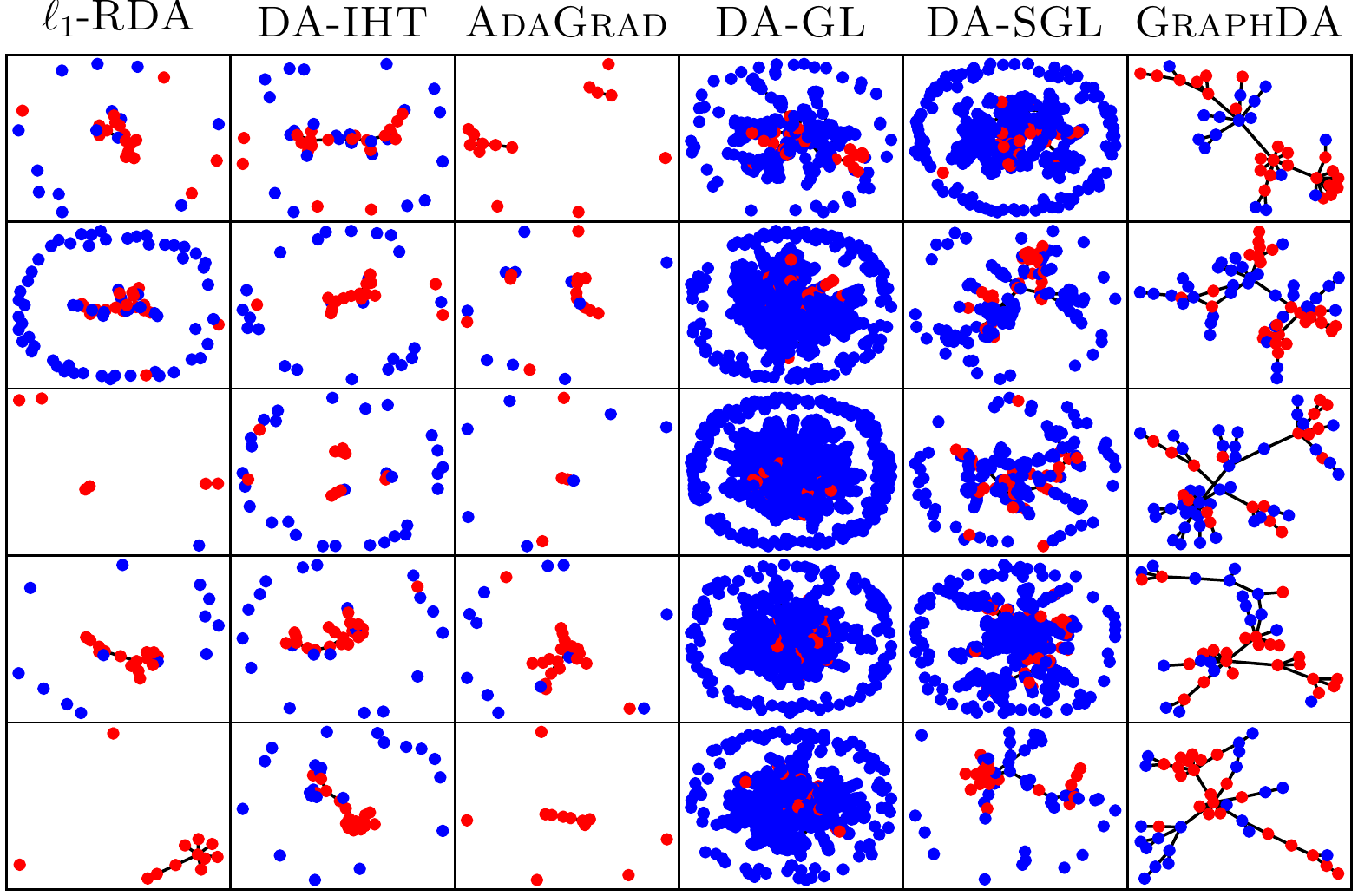}
\caption{HSA05213 pathway detected by different methods. The red nodes are the genes in HSA05213 while blue nodes are the genes not in HSA05213. Results of each row are from a specific trial (from trial 6 to trial 10).}
\label{figure:kegg-learned-genes-rest}
\end{figure}

\end{document}